\documentclass[twoside,11pt]{article}

\usepackage{amsmath}
\usepackage{algorithm}
\usepackage{algpseudocode}
\usepackage{graphics}
\usepackage{epsfig}
\usepackage{color}
\usepackage{multirow}
\usepackage{booktabs}
\usepackage{subfigure}
\usepackage{makecell}
\usepackage{jmlr2e}

\usepackage{lastpage}
\jmlrheading{25}{2024}{1-\pageref{LastPage}}{4/22; Revised
9/24}{11/24}{22-0444}{Jun Chen, Hanwen Chen, Mengmeng Wang, Guang Dai, Ivor W. Tsang and Yong Liu}
\ShortHeadings{Learning Discretized Neural Networks under Ricci Flow}{Chen, Chen, Wang, Dai, Tsang and Liu}

\firstpageno{1}

\begin{document}

\title{Learning Discretized Neural Networks under Ricci Flow}

\author{\name Jun Chen$^{1,2}$ \email junc@zju.edu.cn
\AND
\name Hanwen Chen$^1$ \email chenhanwen@zju.edu.cn 
\AND
\name Mengmeng Wang$^1$ \email mengmengwang@zju.edu.cn
\AND
\name Guang Dai$^3$ \email guang.gdai@gmail.com
\AND
\name Ivor W.~Tsang$^{4,5,6}$ \email ivor.tsang@gmail.com 
\AND
\name Yong Liu$^1$\thanks{Corresponding author.} \email yongliu@iipc.zju.edu.cn
\AND
${}^1$\addr Institute of Cyber-Systems and Control, Zhejiang University, China \\
${}^2$\addr School of Computer Science and Technology, Zhejiang Normal University, China \\
${}^3$\addr SGIT AI Lab, State Grid Corporation of China, China \\
${}^4$\addr Centre for Frontier Artificial Intelligence Research, Agency for Science, Technology and Research (A*STAR), Singapore \\
${}^5$\addr Institute of High Performance Computing, Agency for Science, Technology and Research (A*STAR), Singapore \\
${}^6$\addr College of Computing and Data Science, Nanyang Technological University, Singapore}

\editor{Aurelien Garivier}

\maketitle

\begin{abstract}%   <- trailing '%' for backward compatibility of .sty file
In this paper, we study Discretized Neural Networks (DNNs) composed of low-precision weights and activations, which suffer from either infinite or zero gradients due to the non-differentiable discrete function during training. Most training-based DNNs in such scenarios employ the standard Straight-Through Estimator (STE) to approximate the gradient w.r.t. discrete values. However, the use of STE introduces the problem of gradient mismatch, arising from perturbations in the approximated gradient. To address this problem, this paper reveals that this mismatch can be interpreted as a metric perturbation in a Riemannian manifold, viewed through the lens of duality theory. Building on information geometry, we construct the Linearly Nearly Euclidean (LNE) manifold for DNNs, providing a background for addressing perturbations. By introducing a partial differential equation on metrics, i.e., the Ricci flow, we establish the dynamical stability and convergence of the LNE metric with the $L^2$-norm perturbation. In contrast to previous perturbation theories with convergence rates in fractional powers, the metric perturbation under the Ricci flow exhibits exponential decay in the LNE manifold. Experimental results across various datasets demonstrate that our method achieves superior and more stable performance for DNNs compared to other representative training-based methods.
\end{abstract}


\section{Introduction}

Discretized Neural Networks (DNNs)~\citep{courbariaux2016binarized,li2016ternary,zhu2016trained} have been proven to be efficient in computing, significantly reducing computational complexity, storage space, power consumption, and resources~\citep{chen2021learning}.
Considering a discretized neural network that can be well-trained, the gradient w.r.t. the continuous weight\footnote{In this paper, the continuous weight is relative to the neural network (its data type is full-precision). And the discretized weight is relative to the discretized neural network (its data type is low-precision).} $\boldsymbol{w}$ propagating through a discrete function $Q(\cdot)$, i.e., $\frac{\partial L}{\partial \boldsymbol{w}}=\frac{\partial L}{\partial Q(\boldsymbol{w})} \frac{\partial Q(\boldsymbol{w})}{\partial \boldsymbol{w}}$, suffers from either infinite or zero derivatives because the derivative $\partial Q(\boldsymbol{w}) /\partial \boldsymbol{w}$ is not calculable. In the backward pass, one can obtain the gradient $\partial L /\partial Q(\boldsymbol{w})$, but must update the continuous weight $\boldsymbol{w}$ using the gradient $\partial L /\partial \boldsymbol{w}$. Since the gradient $\partial L /\partial \boldsymbol{w}$ can not be obtained explicitly, the derivative $\partial Q(\boldsymbol{w}) /\partial \boldsymbol{w}$ serves as a bridge to calculate $\partial L /\partial \boldsymbol{w}$ through the standard chain rule.

\begin{figure}[t]
	\centering
	\subfigure[Gradient propagation with STE]
	{\includegraphics[width=.49\textwidth]{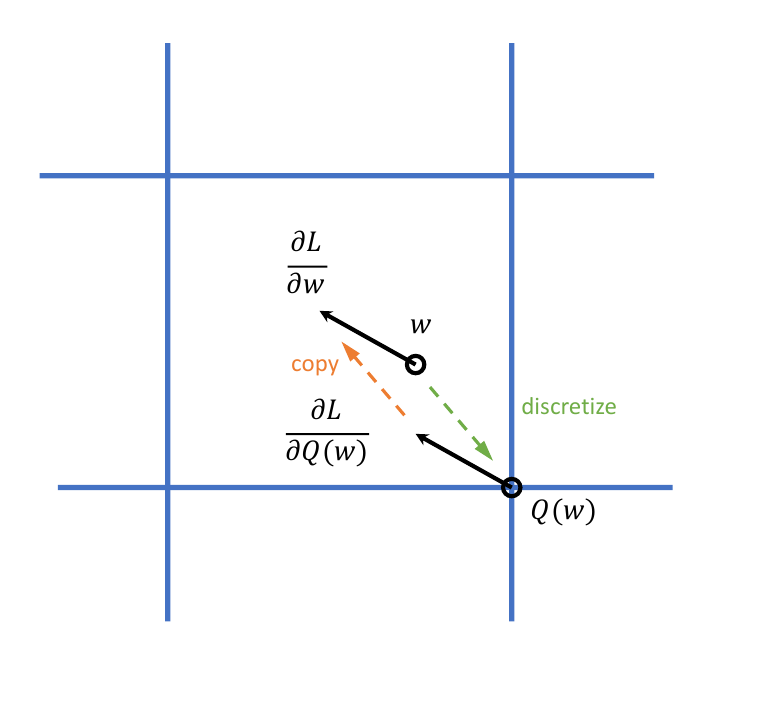}\label{introa}}
	\subfigure[Gradient propagation with metrics]
	{\includegraphics[width=.49\textwidth]{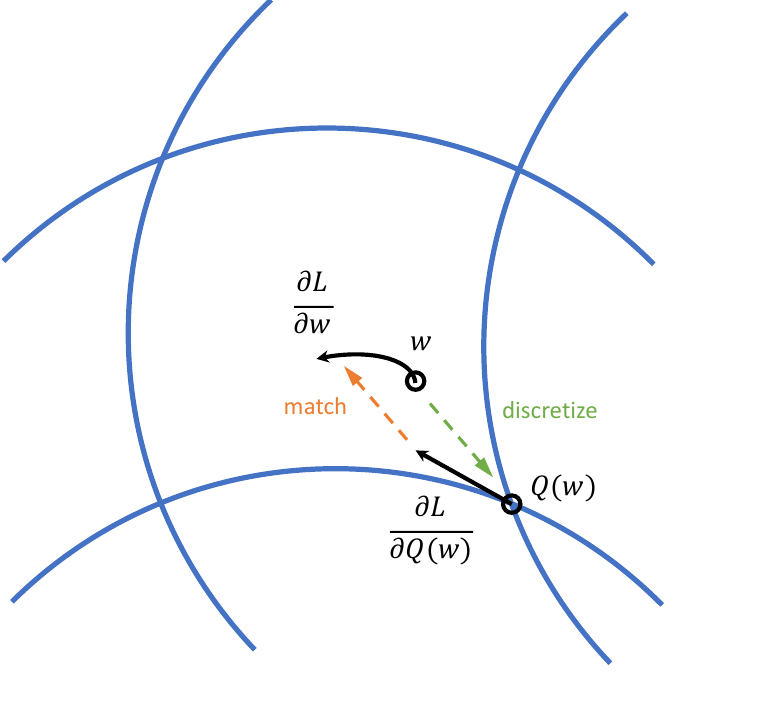}\label{introb}}
	\caption{Comparison of STE and our method. We denote the arrows and points as gradients and weights, respectively. In particular, when a point falls on the grid point, it means that the weight is discretized at this time. In the forward pass, the continuous weight $\boldsymbol{w}$ is mapped to a discrete weight $Q(\boldsymbol{w})$ via a discrete function. In the backward pass, the gradient is propagated from $\partial L /\partial Q(\boldsymbol{w})$ to $\partial L /\partial \boldsymbol{w}$. (a) The STE simply copies the gradient, i.e., $\partial L /\partial \boldsymbol{w}=\partial L /\partial Q(\boldsymbol{w})$. (b) Our method matches the gradient by introducing the proper metric $g_{\boldsymbol{w}}$, i.e., $\partial L /\partial \boldsymbol{w}=g^{-1}_{\boldsymbol{w}}\partial L /\partial Q(\boldsymbol{w})$ in a Riemannian manifold.}
	\label{intro}
\end{figure}

In order to address the problem of either infinite or zero gradients caused by the non-differentiable discrete function, \citet{hinton2012neural} first proposed the concept of Straight-Through Estimator (STE). This estimator 
directly equates $\partial L /\partial \boldsymbol{w}$ and $\partial L /\partial Q(\boldsymbol{w})$ in back-propagation as if the derivative $\partial Q(\boldsymbol{w}) /\partial \boldsymbol{w}$ had been the identity function.
Furthermore, the rigorous definition of STE was developed by \citet{bengio2013estimating}. This definition can be summarized as: the gradient w.r.t. the discretized weight can be approximated by the gradient w.r.t. the continuous weight with clipping, as shown in Figure~\ref{introa}. Subsequently, \citet{courbariaux2016binarized} applied STE to binarized neural networks and provided an approximated gradient as follows:
\begin{equation}
\begin{aligned}
	\frac{\partial L}{\partial \boldsymbol{w}}&=\frac{\partial L}{\partial \operatorname{sign}(\boldsymbol{w})}\mathbb{I}(\boldsymbol{w}), \\
 \operatorname{where}\; \mathbb{I}(w_i):&=\left\{\begin{array}{ll}
		1 & \text { if } \; |w_i| \leq 1 \\
		0 & \text { otherwise}
	\end{array}\right. \operatorname{and}\;
 \operatorname{sign}(w_i):=\left\{\begin{array}{ll}
		+1 & \text { if } \; w_i \geq 0 \\
		-1 & \text { otherwise}
	\end{array}\right. .
	\label{bste}
\end{aligned}
\end{equation}
% where $\mathbb{I}(w_i):=\left\{\begin{array}{ll}
% 		1 & \text { if } \; |w_i| \leq 1 \\
% 		0 & \text { otherwise}
% 	\end{array}\right.$
Clearly, $\mathbb{I}(\cdot)$ is the indicator function, and $\operatorname{sign}(\cdot)$ is the binary function.
Note that the discrete function $Q(\cdot)$ will degenerate to the binary function $\operatorname{sign}(\cdot)$ in binarized neural networks. In this context, $\partial L / \partial \operatorname{sign}(\boldsymbol{w})$ represents the gradient w.r.t. the binarized weight in binarized neural networks.
STE had been successfully implemented in the training of binarized neural networks, and it was further extended to ternary neural networks~\citep{li2016ternary} and arbitrary bit-width discretized neural networks~\citep{zhou2016dorefa}.

In contrast, Non-STE methods encompass all techniques that do not rely on STE, such as those proposed by \citet{hou2016loss}, \citet{bai2018proxquant}, and \citet{leng2018extremely}.
However, the learning process of Non-STE methods is heavily dependent on hyper-parameters~\citep{NEURIPS2019f8e59f4b}, such as weight partition portion in each iteration~\citep{zhou2017incremental} and penalty setting in tuning~\citep{leng2018extremely}.
Consequently, STE methods are widely adopted in DNNs due to their simplicity and versatility.

Nevertheless, the introduction of STE into DNNs inevitably leads to the problem of \emph{gradient mismatch}: the gradient w.r.t. the continuous weight is not strictly equal to the gradient w.r.t. the discretized weight~\citep{NEURIPS2019f8e59f4b}, compromising the training stability of DNNs~\citep{cai2017deep,liu2018bi,qin2020forward}.
Furthermore, the formula of STE indicates that this problem can be alleviated by modifying the gradient.

\citet{zhou2016dorefa} firstly proposed to transform the weight $\boldsymbol{w}$ into the new one $\tilde{\boldsymbol{w}}$ via
\[
	\tilde{\boldsymbol{w}}=\frac{\tanh(\boldsymbol{w})}{\max(|\tanh(\boldsymbol{w})|)}.
\]
By discretizing the new weight $\tilde{\boldsymbol{w}}$, the STE then acts on $\tilde{\boldsymbol{w}}$. During back-propagation, the gradient can be further computed as follows
\[
	\frac{\partial L}{\partial \boldsymbol{w}}=\frac{\partial L}{\partial Q(\tilde{\boldsymbol{w}})} \frac{1-\tanh^2(\boldsymbol{w})}{\max(|\tanh(\boldsymbol{w})|)}.
\]
The authors aim to manually redefine the indicator function $\mathbb{I}(\boldsymbol{w})$ as $\frac{1-\tanh^2(\boldsymbol{w})}{\max(|\tanh(\boldsymbol{w})|)}$. This modification is motivated by the fact that the function $\frac{1-\tanh^2(\boldsymbol{w})}{\max(|\tanh(\boldsymbol{w})|)}$ facilitates a smooth transition, thereby preventing abrupt clipping of the indicator function near $\pm 1$. 
It is remarkable that \citet{NEURIPS2019f8e59f4b} proposed to learn $\partial L /\partial \boldsymbol{w}$ by a neural network, e.g., fully-connected layers or LSTM~\citep{sak2014long}.
Their specific approach is to use neural networks as a shared meta quantizer $M_{\psi}$ parameterized by $\psi$ across layers to replace the gradient via:
\[
	\frac{\partial L}{\partial \boldsymbol{w}}=M_{\psi}\left(\frac{\partial L}{\partial Q(\boldsymbol{w})},\overline{\boldsymbol{w}}\right)\frac{\partial \overline{\boldsymbol{w}}}{\partial \boldsymbol{w}},
\]
where $\overline{\boldsymbol{w}}$ is the weight from the meta quantizer. With the input of the gradient $\partial L /\partial Q(\boldsymbol{w})$, the meta quantizer outputs a new gradient to match $\partial L /\partial \boldsymbol{w}$ by updating the weight $\overline{\boldsymbol{w}}$ in the training process.
Recently, \citet{ajanthan2021mirror} formulated the binarization of neural networks as a constrained optimization problem by introducing a mirror descent framework~\citep{nemirovsky1983informational}. This method performs gradient descent in the dual space (unconstrained space) with gradients computed in the primal space (discrete space). Specifically, by mapping the primal variable $\boldsymbol{w}$ into the dual variable $\tilde{\boldsymbol{w}}=\tanh(\beta_k \boldsymbol{w})$, the gradient can be expressed as
\[
	\frac{\partial L}{\partial \boldsymbol{w}}=\frac{\partial L}{\partial \tilde{\boldsymbol{w}}} \left(1-\tanh^2(\beta_k \boldsymbol{w})\right).
\]
As the hyper-parameter $\beta_k$ approaches infinity, $\tilde{\boldsymbol{w}}$ gradually converges to $\operatorname{sign}(\boldsymbol{w})$ until the corresponding neural network is fully binarized with an adaptive mirror map.

However, the method proposed by \citet{zhou2016dorefa} only avoided abrupt clipping of $\mathbb{I}(\boldsymbol{w})$ by using $\frac{1-\tanh^2(\boldsymbol{w})}{\max(|\tanh(\boldsymbol{w})|)}$, which does not fundamentally alleviate the gradient mismatch in essence. Subsequently, while \citet{NEURIPS2019f8e59f4b} suggested automatically matching the gradient by learning a new neural network (a meta quantizer), it introduces additional errors in the gradient propagation due to extra weights from the meta quantizer, thereby intensifying the problem of gradient mismatch. Furthermore, \citet{ajanthan2021mirror} bypassed the problem of gradient mismatch by directly calculating the derivative $\partial \tilde{\boldsymbol{w}} /\partial \boldsymbol{w}=\left(1-\tanh^2(\beta_k \boldsymbol{w})\right)$, implying that this method does not maintain discrete weights during training. Consequently, the problem of gradient mismatch still remains to be solved.

\subsection{Contributions} 

In this study, we address the gradient mismatch between $\partial L /\partial \boldsymbol{w}$ and $\partial L /\partial Q(\boldsymbol{w})$, treating it as a perturbation phenomenon between these two gradients. By introducing the framework of Riemannian geometry in Figure~\ref{introb}, we further regard the gradient mismatch as a metric perturbation in a Riemannian manifold (Section~\ref{2.2}) through the lens of duality theory~\citep{amari2000methods}. As a partial differential equation on metrics, the Ricci flow~\citep{sheridan2006hamilton}, is introduced, the metric perturbation can be exponentially decayed in theory, providing a solution to the problem of gradient mismatch.
The main contributions of this paper are summarized in the following four aspects:
\begin{itemize}
  \item We propose the LNE manifold endowed with the LNE metric, which is a special form of Ricci-flat metrics in essence. According to the information geometry~\citep{amari2016information}, we construct LNE manifolds for neural networks, providing a background for dealing with perturbations.
  \item We reveal the stability of LNE manifolds under the Ricci-DeTurck flow with the $L^2$-norm perturbation on the basis of the connection between the Ricci-DeTurck flow and the Ricci flow. In this way, any Ricci flow starting close to the LNE metric exists for all time and converges to the LNE metric. This stands in contrast to previous perturbation theories, where the convergence rate is in fractional powers. Instead, the metric perturbation under the Ricci flow exhibits exponential decay in the LNE manifold, providing theoretical support for effectively solving the problem of gradient mismatch.
  \item Based on the appealing characteristics of LNE manifolds under Ricci flow, a novel DNNs with the acceptable complexity, i.e., Ricci Flow Discretized Neural Network (RF-DNN) is developed. In practice, we calculate the Ricci curvature in such a way that the selection of coordinate systems is related to the input transformations of neural networks. In essence, the discrete Ricci flow is employed to overcome the problem of gradient mismatch.
  \item The experiments are implemented on several classification benchmark datasets and network structures. Experimental results demonstrate the effectiveness of RF-DNN compared with other representative training-based methods.
\end{itemize}

\subsection{Overall Organization}

The paper is organized as follows. In Section~\ref{2}, we introduce the motivation and Ricci flow. Section~\ref{3} deduces the corresponding LNE manifold for neural networks based on the geometric structure measured by the LNE divergence. The stability of LNE manifolds under the Ricci-DeTurck is proved in Section~\ref{4}. In Section~\ref{5}, we calculate the approximated gradient in the LNE manifold to avoid solving the inverse of the LNE metric. Section~\ref{6} presents how to introduce discrete Ricci flow into DNNs and yields the corresponding algorithm. The experimental results and ablation studies for RF-DNNs are presented in Section~\ref{7}. Section~\ref{8} concludes the entire paper. Proofs are provided in the Appendices.

The Ricci flow on Ricci-flat metrics is known in the literature to be stable for $C^0$ perturbations in the $L^{\infty}$-norm (Section~\ref{2.4}).
Based on a Bregman divergence~\citep{bregman1967relaxation}, the LNE metric, a special form of Ricci-flat metrics, is introduced in neural networks via the LNE divergence (Theorem~\ref{thm7}). The stability of LNE manifolds under the Ricci-DeTurck flow is then proved (Corollary~\ref{cor2} and Theorem~\ref{thm6}). A discretization of the Ricci flow is therefore proposed, leading to a practical algorithm (RF-DNNs, Algorithm~\ref{alg2}).

\section{Motivation and Formulation}
\label{2}

\subsection{Background}
\label{2.1}

To establish the foundation for our study throughout the paper, we begin with the basic background for feed-forward DNNs, drawing from the work by \citet{martens2015optimizing}. Important notations are listed in Appendix~\ref{appnotation}.

A neural network can be regarded as a function that transforms the input $\boldsymbol{a}_0$ into the output $\boldsymbol{a}_l$ through a series of $l$ layers. For the $i$-th layer ($i \in \{1,2,\ldots,l\}$), we denote $\boldsymbol{W}_i$ as the weight matrix, $\boldsymbol{s}_i$ as the vector of these weighted sum, and $\boldsymbol{a}_i$ as the vector of output (also known as the activation). Each layer receives vectors of a weighted sum of the input from the previous layer and calculates their output through a nonlinear function. Note that we ignore the bias vector for brevity.

For a DNN, the introduction of a discrete function $Q(\cdot)$ is necessary to discretize the weight matrix $\boldsymbol{W}_i$ and the activation vector $\boldsymbol{a}_i$. We denote the discretized weight matrix as $\boldsymbol{\hat{W}}_i=Q(\boldsymbol{W}_i)$ and the discretized activation vector as $\boldsymbol{\hat{a}}_i=Q(\boldsymbol{a}_i)$. Then, the feed-forward of DNNs at each layer is given as follows:
\begin{equation}
  \begin{aligned}
    &\boldsymbol{s}_i = \boldsymbol{\hat{W}}_i \boldsymbol{\hat{a}}_{i-1} \\
    &\boldsymbol{a}_i = f(\boldsymbol{s}_i) \\
    &\boldsymbol{\hat{a}}_i=Q(\boldsymbol{a}_i)
    \label{eq1}
  \end{aligned}
\end{equation}
where $f$ is a nonlinear (activation) function. The vectorized weights in each layer, before and after discretization, are denoted as $\boldsymbol{w}$ and $\boldsymbol{\hat{w}}$, respectively.
Additionally, we define the discretized parameter vector as $\boldsymbol{\hat{\xi}}=\left[\operatorname{vec}\left(\boldsymbol{\hat{W}}_{1}\right)^{\top}, \operatorname{vec}\left(\boldsymbol{\hat{W}}_{2}\right)^{\top}, \ldots, \operatorname{vec}\left(\boldsymbol{\hat{W}}_{l}\right)^{\top}\right]^{\top}$, which consists of all of the network's parameters concatenated together, where $\operatorname{vec}(\cdot)$ is the operator that vectorizes a matrix by stacking their columns together.
Similarly, the parameter vector is defined as $\boldsymbol{\xi}=\left[\operatorname{vec}\left(\boldsymbol{W}_{1}\right)^{\top}, \operatorname{vec}\left(\boldsymbol{W}_{2}\right)^{\top}, \ldots, \operatorname{vec}\left(\boldsymbol{W}_{l}\right)^{\top}\right]^{\top}$. Details regarding the back-propagation of DNNs are provided in later sections.

% In frequentist statistics, we represent the loss function as the negative log-likelihood w.r.t. discretized parameter $\boldsymbol{\hat{\xi}}$, where the input $\boldsymbol{a}_0$ can be observed. The way minimizing the loss function is to maximize the likelihood:
% \begin{equation}
%   L(\boldsymbol{y},\boldsymbol{z}) = - \log p\left(\boldsymbol{y}|\boldsymbol{z},\boldsymbol{\hat{\xi}}\right)
%   \label{eq2}
% \end{equation}
% where $\boldsymbol{y}$ is the target and $\boldsymbol{z}$ is the prediction calculated by DNNs. Note that $p\left(\boldsymbol{y}|\boldsymbol{z},\boldsymbol{\hat{\xi}}\right)$ is the probability density function defined by the conditional probability distribution $P_{\boldsymbol{y}|\boldsymbol{z}}\left(\boldsymbol{\hat{\xi}}\right)$ of DNNs. 

\subsection{Motivation}
\label{2.2}

We consider that the source of the gradient mismatch lies in a perturbation phenomenon between $\partial L /\partial \boldsymbol{w}$ and $\partial L /\partial Q(\boldsymbol{w})$ in terms of linear operators, expressed as:
\begin{equation}
	\frac{\partial L}{\partial \boldsymbol{w}}=\frac{\partial L}{\partial Q(\boldsymbol{w})} + \mathcal{P}\left(\frac{\partial L}{\partial Q(\boldsymbol{w})}\right),
	\label{pertur}
\end{equation}
where the perturbation function $\mathcal{P}$ takes the gradient $\partial L /\partial Q(\boldsymbol{w})$ as input, with $\mathcal{P}\left(\partial L /\partial Q(\boldsymbol{w})\right)$ being much smaller than $\partial L /\partial Q(\boldsymbol{w})$. In general, $\mathcal{P}\left(\partial L /\partial Q(\boldsymbol{w})\right)$ can be expressed as $\mathcal{P}\left(\partial L /\partial Q(\boldsymbol{w})\right)=o\left(\partial L /\partial Q(\boldsymbol{w})\right)$.
If the perturbation term $\mathcal{P}\left(\partial L /\partial Q(\boldsymbol{w})\right)$ can be significantly eliminated or decayed, an elegant solution to the gradient mismatch arises. Within the framework of perturbation theory in linear spaces~\citep{kato2013perturbation}, the rate of convergence for perturbations is typically expressed in fractional powers.

Inspired by the mirror descent framework\footnote{Mirror descent induces non-Euclidean structure by solving iterative optimization problems using different proximity functions. This algorithm is introduced by~\citet{nemirovsky1983informational}, and analyzed by~\citet{beck2003mirror}.}, one can map the parameter from the primal space to the dual space, and subsequently calculate the gradient in the dual space. Naturally\footnote{Natural gradient descent selects the steepest descent along a Riemannian manifold by multiplying the standard gradient by the inverse of the metric tensor~\citep{amari1998natural}.It is worth mentioning that mirror descent and natural gradient descent are proven to be equivalent~\citep{raskutti2015information}, implying that mirror descent represents the steepest descent direction along the Riemannian manifold corresponding to the choice of Bregman divergence.}, when the Riemannian metric structure is introduced by means of information geometry, the gradient mismatch is conclusively viewed as a metric perturbation in a Riemannian manifold. Specifically, we rewrite the gradient $\partial L /\partial \boldsymbol{w}$ in Euclidean space as the gradient $\tilde{\partial} L /\tilde{\partial} \boldsymbol{w}$ in a Riemannian manifold. For the sake of simplicity, we use ``$\tilde{\partial}$" to denote the derivative in a Riemannian manifold and ``$\partial$" to denote the derivative in Euclidean space. The difference between these two gradients is governed by the inverse of the corresponding metric tensor.
% \footnote{\textcolor{red}{In this paper, we regulate that when we use ``$\tilde{\partial}$", it indicates that the gradient at this time is in a Riemannian manifold, so as to distinguish it from the gradient in Euclidean space when we use ``$\partial$". And the difference between these two gradients is the inverse of the corresponding metric tensor.}}. 
The problem of gradient mismatch can be further expressed as:
\begin{equation}
	\label{metripertur}
	\frac{\tilde{\partial} L}{\tilde{\partial} \boldsymbol{w}}=g^{-1}_{\boldsymbol{w}}\frac{\partial L}{\partial Q(\boldsymbol{w})},
\end{equation}
%\begin{equation}
	%\label{metripertur}
	%\frac{\tilde{\partial} L}{\tilde{\partial} %\boldsymbol{w}}=g^{-1}_{\boldsymbol{w}}\frac{\partia%l L}{\partial %\boldsymbol{w}}=g^{-1}_{\boldsymbol{w}}\frac{\partia%l L}{\partial Q(\boldsymbol{w})}
%\end{equation}
where the perturbation item is implied in the metric $g_{\boldsymbol{w}}$, with the term $\frac{\partial L}{\partial Q(\boldsymbol{w})}$ representing the gradient $\partial L(\boldsymbol{w})$ as defined in Definition~\ref{def7}. Then the metric perturbation emerges, and the perturbation at this time is referred to the deviation from the original metric. In this way, we present the generalization of STE in a Riemannian manifold, which will degenerate into the standard STE when the Riemannian metric $g$ returns to the Euclidean metric $\delta$.

\begin{definition}
	\label{def7}
	\citep{amari1998natural} The steepest descent direction of $L(\boldsymbol{w})$ in a Riemannnian manifold, i.e., the \textbf{natural gradient descent}, is given by
	\[
	\tilde{\partial} L(\boldsymbol{w})=g^{-1}_{\boldsymbol{w}} \partial L(\boldsymbol{w}),
	\]
	where $g^{-1}=(g^{ij})$ is the inverse of the metric $g=(g_{ij})$ and $\partial L(\boldsymbol{w})$ is the gradient:
	\[
	\partial L(\boldsymbol{w}) = \left[\frac{\partial L(\boldsymbol{w})}{\partial w_1},\cdots,\frac{\partial L(\boldsymbol{w})}{\partial w_n}\right]^{\top} .
	\]
\end{definition}

Subsequently, a key question arises: What kind of manifolds do we need to construct to naturally and effectively handle metric perturbations? Or, what makes a manifold ``good" in the presence of perturbations? In practice, general relativity gives an excellent example in nature of dealing with small gravitational perturbations within the framework of a Riemannian manifold~\citep{wald2010general}.
To address the approximation in scenarios where gravity is ``weak", the spacetime metric is nearly flat at this time in the context of general relativity. This approximation is sufficient for most cases, except for phenomena involving gravitational collapse and the large-scale structure of the universe. Assuming that the deviation $\gamma_{ij}$ of the actual spacetime metric $g_{ij}= \eta_{ij}+ \gamma_{ij}$ from a flat metric $\eta_{ij}$ is ``small", the linearized gravity is introduced to approximate the gravity in general relativity\footnote{Firstly, when analyzing flat spacetime corresponding to ``zero" gravity, the flat metric $\eta_{ij}$ can be employed. Secondly, when dealing with nearly flat spacetime indicative of ``weak" gravity, one can use the nearly flat metric $g_{ij}= \eta_{ij}+ \gamma_{ij}$ for analysis. In this case, $g_{ij}$ and $\eta_{ij}$ are very close, allowing the first-order Taylor expansion of this linearized form to yield sufficiently accurate results. For instance, this metric form can accurately analyze the gravity produced by celestial bodies like the Earth or the Sun. Thirdly, when faced with ``strong" gravity resulting from the large-scale structure of the universe, the linearized metric is no longer applicable due to the curved nature of spacetime.}. In this context, ``smallness" is defined such that the components of $\gamma_{ij}$ are much smaller than $1$ in the global inertial coordinate system of $\eta_{ij}$. Such a linearly nearly flat metric greatly simplifies the calculation of ``weak" gravity, and manifolds constructed with such metrics are considered sufficient for approximating the manifold with perturbations.

\begin{figure}[htbp]
	\centering
	\includegraphics[width=\textwidth]{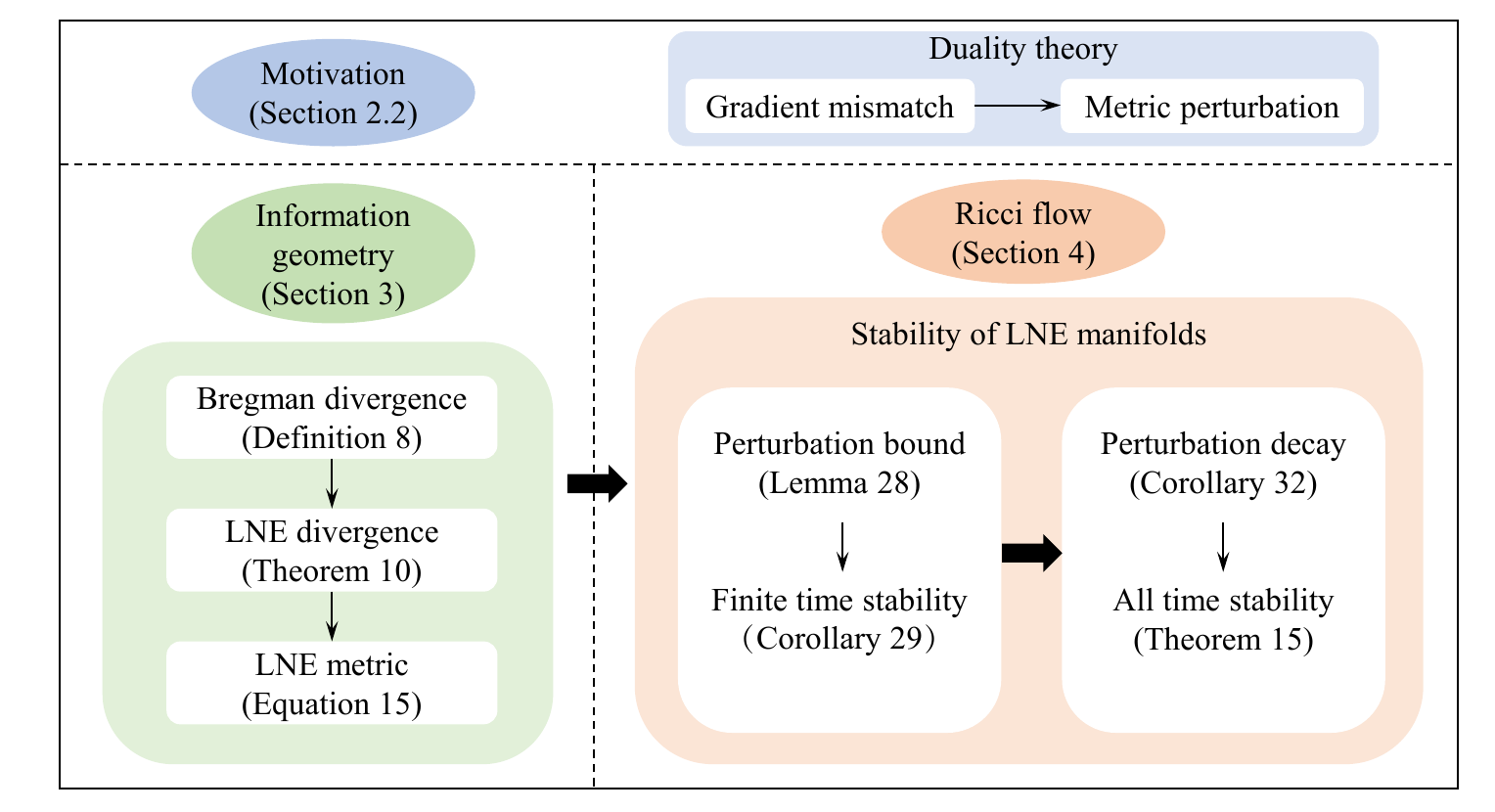}
	\caption{The overview of the theoretical ideas.}
	\label{overview}
\end{figure}

Similarly, in this paper, we define the Linearly Nearly Euclidean (LNE) metric by regarding the Euclidean metric $\delta_{ij}$ as the flat metric $\eta_{ij}$. This metric plays a crucial role in handling metric perturbations in the background of LNE manifolds for DNNs. Motivated by the natural gradient descent connecting a neural network with the Riemannian metric\footnote{In the natural gradient descent, the Riemannian metric is expressed in the form of Fisher information matrix, i.e., $g=\begin{bmatrix} E\left[\operatorname{vec}\left(\frac{\partial L}{\partial \boldsymbol{W}_1}\right) \operatorname{vec}\left(\frac{\partial L}{\partial \boldsymbol{W}_1}\right)^\top\right] & \cdots & E\left[\operatorname{vec}\left(\frac{\partial L}{\partial \boldsymbol{W}_1}\right) \operatorname{vec}\left(\frac{\partial L}{\partial \boldsymbol{W}_l}\right)^\top\right] \\ \vdots&\ddots&\vdots \\ E\left[\operatorname{vec}\left(\frac{\partial L}{\partial \boldsymbol{W}_l}\right) \operatorname{vec}\left(\frac{\partial L}{\partial \boldsymbol{W}_1}\right)^\top\right]&\cdots&E\left[\operatorname{vec}\left(\frac{\partial L}{\partial \boldsymbol{W}_l}\right) \operatorname{vec}\left(\frac{\partial L}{\partial \boldsymbol{W}_l}\right)^\top\right] \end{bmatrix}$, where Fisher information matrix is associated with the weights from a neural network~\citep{martens2015optimizing}.}, LNE metrics can be mathematically constructed in neural networks. To achieve this, our method involves introducing a convex function to derive the LNE divergence with the assistance of Bregman divergence~\citep{bregman1967relaxation}. The transition from a convex function to the LNE divergence operates within the mirror descent framework. Subsequently, the step from the LNE divergence to the LNE metric incorporates the concept of information geometry. Consequently, the LNE metric emerges in the gradient of the LNE manifold, similar to Definition~\ref{def7}. 
%So far, all preparations have been completed, including the selection and construction of the manifold, and 
Finally, with the constructed manifold for DNNs in place, the remaining problem is how to efficiently decay the metric perturbation. This is achieved by employing a geometric tool, i.e., Ricci flow.

In addition,  a series of proofs about stability illustrates that the Ricci flow can decay the metric perturbation in the cases of Ricci-flat metrics. Therefore, as long as we can prove that a small perturbation of the LNE metric under the Ricci flow decays, the metric perturbation can be alleviated, providing a theoretical solution for the problem of gradient mismatch in the training of DNNs. In contrast to previous perturbation theories, where the convergence rate is in fractional powers, the metric perturbation under the Ricci flow can be exponentially decayed in the LNE manifold. Figure~\ref{overview} give an overview of the theoretical ideas to facilitate sorting out the solution steps.

\subsection{Ricci Flow}
\label{2.3}

\begin{definition}
	\label{def8}
	\citep{sheridan2006hamilton} A \textbf{Riemannian metric} on a smooth manifold $\mathcal{M}$ is a smoothly-varying inner product on the tangent space $T_p \mathcal{M}$ at each point $p \in \mathcal{M}$, i.e., a (0,2)-tensor which is symmetric and positive-definite at each point of $\mathcal{M}$. One will usually write $g$ for a Riemannian metric, and $g_{ij}$ for it coordinate representation. A manifold together with a Riemannian metric, $(\mathcal{M},g)$, is called a \textbf{Riemannian manifold}.
\end{definition}

The concept of Ricci flow was first proposed by Hamilton~\citep{hamilton1982three} on the Riemannian manifold $\mathcal{M}$, building upon Definition~\ref{def8} for a time-dependent metric $g(t)$. Given the initial metric $g_0$, the Ricci flow is described by a partial differential equation that evolves the metric tensor:
\begin{equation}
\begin{aligned}
\frac{\partial}{\partial t} g(t) &=-2 \operatorname{Ric}(g(t)) \\
g(0) &=g_{0}
\end{aligned}
\label{ricci}
\end{equation}
where $\operatorname{Ric}$ denotes the Ricci curvature tensor, with a detailed definition available in Appendix~\ref{app1}.
The purpose of the Ricci flow is to prove Thurston's Geometrization Conjecture and Poincar\'e Conjecture, guiding the evolution of the metric towards specific geometric structures and topological properties~\citep{sheridan2006hamilton}.

\begin{corollary}
	\label{def1}
	\citep{sheridan2006hamilton} The Ricci flow is \textbf{strongly parabolic} if there exists $\delta > 0$ such that for all covectors $\varphi \neq 0$ and all (symmetric\footnote{The Riemannian metric $g_{ij}$ is always symmetric based on Definition~\ref{def8}. Hence, $h_{ij}=\frac{\partial}{\partial t} g_{ij}(t)$ is required to be symmetric.}) $h_{ij}=\frac{\partial}{\partial t} g_{ij}(t) \neq 0$, the principal symbol of the differential operator $-2\operatorname{Ric}$ satisfies
	\[
	[-2\operatorname{Ric}](\varphi)(h)_{ij} h^{ij}=g^{p q}\left(\varphi_{p} \varphi_{q} h_{i j}+\varphi_{i} \varphi_{j} h_{p q}-\varphi_{q} \varphi_{i} h_{j p}-\varphi_{q} \varphi_{j} h_{i p}\right) h^{i j}>\delta \varphi_{k} \varphi^{k} h_{r s} h^{r s}
	\]
    where $h^{ij}$ is the inverse of $h_{ij}$.
\end{corollary}

\begin{theorem}
	\label{thm1}
	\citep{ladyzhenskaia1988linear} Suppose that $u(t): \mathcal{M}\times[0, T) \rightarrow \mathcal{E}$ is a time-dependent section of the vector bundle $\mathcal{E}$ where $\mathcal{M}$ is a Riemannian manifold. If the system of the Ricci flow is strongly parabolic at $u_0$ where $u_0=u(0):\mathcal{M} \rightarrow \mathcal{E}$, then there exists a unique solution on the time interval $[0, T)$.
\end{theorem}

Combined with Corollary~\ref{def1} and Theorem~\ref{thm1}, one can determine the existence of a unique solution of the Ricci flow over a short time by verifying whether it is strongly parabolic.
However, if we choose $h_{ij}=\varphi_i \varphi_j$, it is clear that the left hand side of the inequality in Corollary~\ref{def1} is $0$, thus the inequality can not hold. As a consequence, Ricci flow is not always strongly parabolic, and this lack of guarantee for the existence of a solution is highlighted by Theorem~\ref{thm1}. In the following analysis, we delve into the non-parabolic nature and find a solution based on the relationship between the Ricci flow and the Ricci-DeTurck flow. The impact of its non-parabolic nature on different parts can be understood through the linearization of the Ricci curvature tensor. We define the linearization of the Ricci curvature as $\mathcal{D}[\operatorname{Ric}]$ such that
\[
\mathcal{D}[\operatorname{Ric}]\left(\frac{\partial}{\partial t} g_{ij}(t)\right) = \frac{\partial}{\partial t} \operatorname{Ric}(g_{ij}(t)).
\]

\begin{lemma}
	\label{le1}
	The linearization of $-2\operatorname{Ric}$ can be rewritten as\footnote{In this paper, we use the Einstein summation convention (for example, $(AB)^j_i=A^k_i B^j_k$). When the same index appears twice in one term, once as an upper index and the other time as a lower index, summation is automatically taken over this index even without the summation symbol.}
	\begin{equation}
	\begin{aligned}
	&\mathcal{D}[-2 \operatorname{Ric}](h)_{i j}=g^{p q} \nabla_{p} \nabla_{q} h_{i j}+\nabla_{i} V_{j}+\nabla_{j} V_{i}+O(h_{ij}) \\
	&\operatorname{where} \;\;\;\;V_{i}=g^{p q}\left(\frac{1}{2} \nabla_{i} h_{p q}-\nabla_{q} h_{p i}\right) \operatorname{and} \;\; h_{ij}=\frac{\partial}{\partial t} g_{ij}(t).
	\end{aligned}
	\end{equation}
\end{lemma}

\begin{proof}
	The proofs can be found in Appendix~\ref{app21}.
\end{proof}

By carefully observing Lemma~\ref{le1}, the impact on the non-parabolic nature of the Ricci flow comes from the terms $V_i$ and $V_j$~\citep{sheridan2006hamilton}, rather than the term $g^{p q} \nabla_{p} \nabla_{q} h_{i j}$. On the other hand, the term $O(h_{ij})$ will have no contributions to the principal symbol of $-2 \operatorname{Ric}$, so we can ignore it in this problem. Next, we attempt to eliminate the impact of the non-parabolic nature on the Ricci flow.

Using a time-dependent diffeomorphism $\varphi(t): \mathcal{M}\rightarrow\mathcal{M}$ (with $\varphi(0)=\operatorname{id}$), the pullback metrics $g(t)$ can be expressed as
\begin{equation}
g(t)=\varphi^*(t) \bar{g}(t),
\label{pb}
\end{equation}
satisfying the Ricci flow equation, where $\varphi^*(t)$ is the pullback through $\varphi(t)$. The above formula yields the new metric $\bar{g}(t)$ via the pullback, and the terms $V_i$ and $V_j$ can be reparameterized by choosing $\varphi(t)$ to form the Ricci-DeTurck flow (w.r.t. $\bar{g}(t)$), which is strongly parabolic.
Furthermore, the solution is followed by the DeTurck Trick~\citep{deturck1983deforming}, involving a time-dependent reparameterization of the manifold:
\begin{equation}
\begin{aligned}
\frac{\partial}{\partial t}\bar{g}(t)&=-2 \operatorname{Ric}(\bar{g}(t))- \mathcal{L}_{\frac{\partial \varphi(t)}{\partial t}} \bar{g}(t) \\
\bar{g}(0) &=\bar{g}_0,
\end{aligned}
\label{deturck}
\end{equation}
See Appendix~\ref{app22} for details. Thus, the Ricci-DeTurck flow has a unique solution for a short time. For the long time behavior, please refer to Appendix~\ref{app24}.

\subsection{Literature}
\label{2.4}

For the Riemannian $n$-dimensional manifold $(\mathcal{M}^n,g)$ that is isometric to the Euclidean $n$-dimensional space $(\mathbb{R}^n,\delta)$, \citet{schnurer2007stability} showed the stability of the Euclidean space under the Ricci flow for a small $C^0$ perturbation.
\citet{koch2012geometric} demonstrated the stability of the Euclidean space along with the Ricci flow in the $L^{\infty}$-norm. Moreover, for the decay of the  $L^{\infty}$-norm on Euclidean space, \citet{appleton2018scalar} provided a proof from another idea.

On the other hand, for a Ricci-flat metric with small perturbations, \citet{guenther2002stability} proved that such metrics converge under Ricci flow.
Considering the stability of integrable and closed Ricci-flat metrics, \citet{sesum2006linear} proved that the convergence rate is exponential because the spectrum of the Lichnerowicz operator is discrete. Furthermore, \citet{deruelle2021stability} demonstrated that an asymptotically locally Euclidean Ricci-flat metric is dynamically stable under the Ricci flow, with the $L^2 \cap L^{\infty}$ perturbation on non-flat and non-compact Ricci-flat manifolds. In our work, we discuss aspects related to Ricci-flat manifolds.

\section{Neural Networks in LNE Manifolds}
\label{3}

The aim of this section is to bulid an LNE manifold via information geometry, laying the foundation for the introduction of the Ricci flow. Specifically, we first introduce a convex function (Equation~(\ref{convex1})) to derive the LNE divergence (Theorem~\ref{thm7}) with the assistance of Bregman divergence (Definition~\ref{def6}). We then construct the LNE metric (Equation~(\ref{lne})) by incorporating the LNE divergence into neural networks. Consequently, the LNE metric emerges in the steepest descent gradient (Lemma~\ref{lem5}) of the LNE manifold.
Certainly, the mirror descent algorithm can equivalently establish the link between divergences and gradients, but it lacks the geometric meaning (manifold and metric) crucial for our purposes.

\subsection{Neural Network Manifold}
\label{3.1}

A neural network is composed of a large number of neurons connected with each other. The set of all such networks forms a manifold, where the weights represented by the neuron connections can be regarded as the coordinate system.

% For the $n \times n$ matrix, all such matrices form an $n^2$-dimensional manifold. Specifically, \textcolor{red}{all symmetric and positive-definite matrices form the $n(n+1)/2$-dimensional manifold, i.e., a submanifold embedded in the manifold of all the matrices. Note that the components of a metric in a coordinate basis take on the form of a symmetric and positive-definite matrix in differential geometry~\citep{helgason2001differential}.} Since symmetric and positive-definite matrices have many advantages, which lead to wide implementations in real-world applications, our method will also construct such symmetric and positive-definite matrices (metrics) with some appealing geometric characteristics in neural networks.

\begin{remark}
Comparing straight lines in Euclidean space, geodesics are the straightest possible lines that we can draw in a Riemannian manifold. Given a geodesic, there exists a unique non-Euclidean coordinate system. Once a curved coordinate system is selected in a Riemannian manifold, the symmetric and positive-definite metric is also defined based on Definition~\ref{def8}. This geometry-based metric can describe the properties of manifolds, such as curvature~\citep{helgason2001differential}.
\end{remark}

\subsection{Euclidean Space and Divergence}
\label{3.2}

From the viewpoint of information geometry, the metric can be deduced by the divergence satisfying the certain criteria~\citep{basseville2013divergence}, summarized in Definition~\ref{def5}. 
\begin{definition}
	\label{def5}
	\citep{amari2016information} $D[P:Q]$ is called a \textbf{divergence} when it satisfies the following criteria:

	(1) $D[P:Q] \geq 0$,

	(2) $D[P:Q]=0$ when and only when $P=Q$,

	(3) When $P$ and $Q$ are sufficiently close to each other, and their coordinates are denoted by $\boldsymbol{\xi}_P$ and $\boldsymbol{\xi}_Q=\boldsymbol{\xi}_P+d\boldsymbol{\xi}$ respectively, the Taylor expansion of the divergence can be written as
	\begin{equation}
	D[\boldsymbol{\xi}_P:\boldsymbol{\xi}_Q]=\frac{1}{2}\sum_{i,j} g_{ij}(\boldsymbol{\xi}_P) d\xi_i d\xi_j + O(|d\boldsymbol{\xi}|^3),
  \end{equation}
	and the Riemannian metric $g_{ij}$ is symmetric and positive-definite\footnote{The components of a Riemannian metric in a coordinate basis take on the form of a symmetric and positive-definite matrix in differential geometry~\citep{helgason2001differential}.}, acting on $\boldsymbol{\xi}_P$.
\end{definition}

When $P$ and $Q$ are sufficiently close, expressed in coordinates as column vectors $\boldsymbol{\xi}_P$ and $\boldsymbol{\xi}_Q$ based on Definition~\ref{def5}, the square of an infinitesimal distance $ds^2$ between them can be defined as:
\begin{equation}
\label{ds}
  ds^2 = 2D[\boldsymbol{\xi}_P:\boldsymbol{\xi}_Q]=\sum_{i,j} g_{ij}(\boldsymbol{\xi}_P) d\xi_i d\xi_j
\end{equation}
where $d\boldsymbol{\xi}$ denotes a sufficiently small coordinate variation between the coordinates $\boldsymbol{\xi}_P$ and $\boldsymbol{\xi}_Q$. Here, we can ignore the third-order term $O(|d\boldsymbol{\xi}|^3)$ followed by~\citet{amari2016information} because the second-order approximation can give sufficiently accurate results. A manifold $\mathcal{M}$ is said to be Riemannian when a postive-definite metric $g_{ij}$ is defined on $\mathcal{M}$, and the square of the local distance between $\boldsymbol{\xi}_P$ and $\boldsymbol{\xi}_Q$ is given by Equation~(\ref{ds}). Geometrically, the divergence $D[\boldsymbol{\xi}_P:\boldsymbol{\xi}_Q]$ provides the manifold with a Riemannian structure.

Using an orthonormal Cartesian coordinate system in Euclidean space, the Euclidean divergence is defined as half of the square of the Euclidean distance between $\boldsymbol{\xi}$ and $\boldsymbol{\xi}'$
\begin{equation}
D_E [\boldsymbol{\xi}:\boldsymbol{\xi}']=\frac{1}{2}\sum_{i} (\xi_i-\xi'_i)^2.
\label{eq18}
\end{equation}
In this context, the Riemannian metric $g_{ij}$ degenerates into the Euclidean metric $\delta_{ij}$, resulting in the squared infinitesimal distance $ds^2$ expressed as:
\begin{equation}
  ds^2=2D_{E}[\boldsymbol{\xi}:\boldsymbol{\xi}+d\boldsymbol{\xi}]
  =\sum_i (d\xi)^2=\sum_{i,j} \delta_{ij} d\xi_i d\xi_j.
\end{equation}
It is worth noting that the Euclidean metric $\delta_{ij}$ is equivalent to the identity matrix $\boldsymbol{I}$, and we use the notation of metrics here for consistency with geometry theory conventions.

\subsection{LNE Manifold and Divergence}
\label{3.3}

Recall that\footnote{The link between the LNE manifold and general relativity can be found in Section~\ref{2.2}.}, in general relativity~\citep{wald2010general}, the complete Riemannian manifold $(\mathcal{M},g)$ endowed with a linearly nearly flat spacetime metric is considered to address the Newtonian limit through the linearized gravity. The form of this metric is $g_{ij}= \eta_{ij}+ \gamma_{ij}$, where $\eta_{ij}$ represents the Minkowski metric (background to special relativity in flat spacetime), and $\gamma_{ij}$ denotes small perturbations. In practice, this theory is excellent for describing small gravitational perturbations when gravity is ``weak".

Similarly, we define a metric $g_{ij}= \delta_{ij}+ \gamma_{ij}$ in a Riemannian manifold, where $\delta_{ij}$ represents a flat Euclidean metric. An adequate definition of ``smallness" in this context is that the components of $\gamma_{ij}$ are much smaller than $1$ in the global inertial coordinate system of $\delta_{ij}$.
Therefore, we can systematically develop the LNE metric to address small perturbations.

\subsubsection{Convex Function and Bregman Divergence}

To construct the LNE manifold endowed with the LNE metric in the neural network, in accordance with Definition~\ref{def5}, we introduce a divergence to express the LNE metric, drawing an analogy to the relationship between the Euclidean metric and its divergence.

\begin{figure}[thbp]
	\centering
	\includegraphics[width=.7\textwidth]{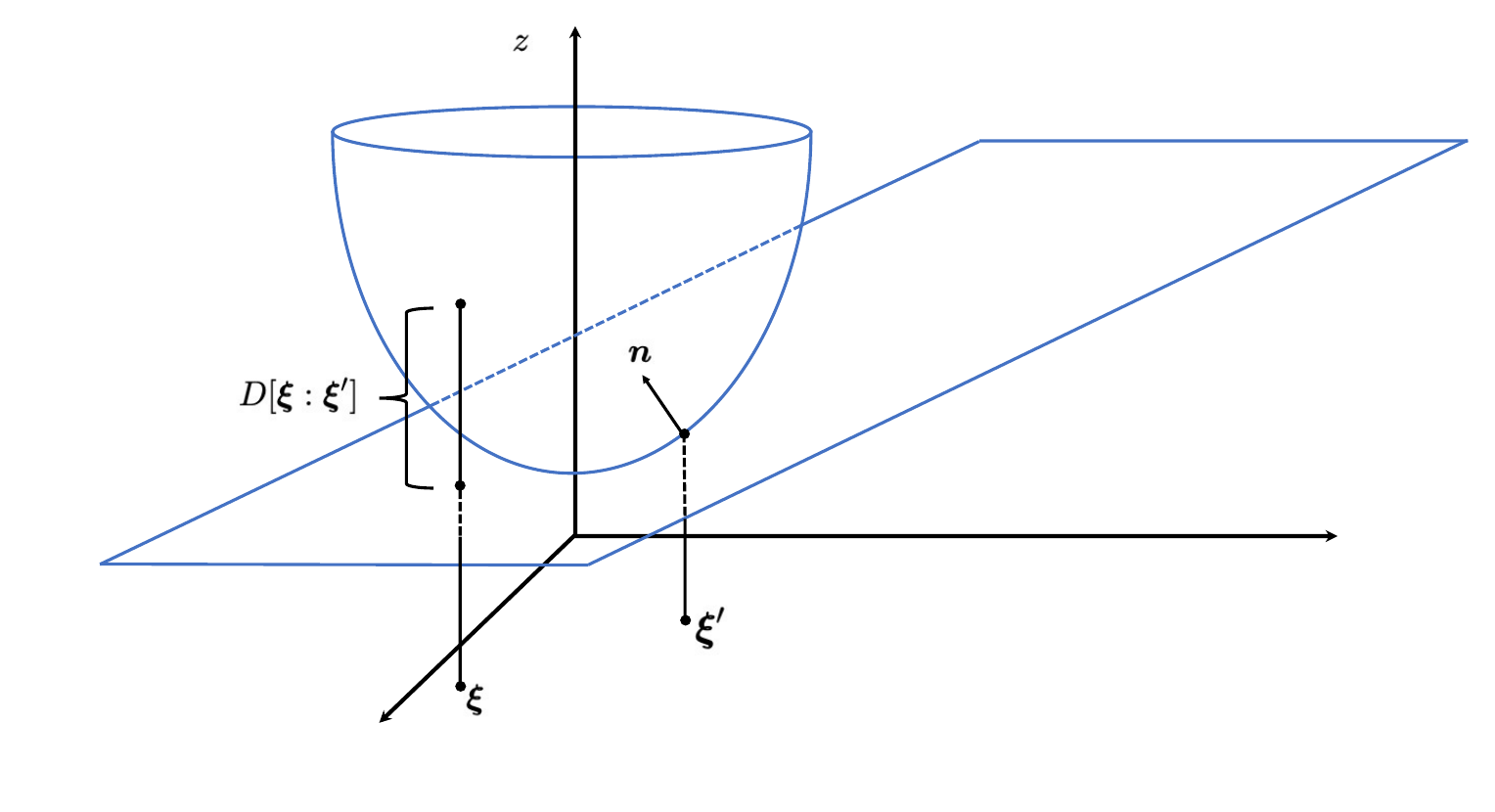}
	\caption{The divergence $D[\boldsymbol{\xi}:\boldsymbol{\xi}']$ is viewed as the distance between the convex function $\Phi(\boldsymbol{\xi})$ and its tangent hyperplane $z$, where the supporting hyperplane with normal vector $\boldsymbol{n}=\nabla \Phi(\boldsymbol{\xi}')$ at the point $\boldsymbol{\xi}'$ is defined.}
	\label{convex}
\end{figure}

The construction of a divergence relies on finding a suitable convex function~\citep{bubeck2015convex}. Here, we introduce a nonlinear function $\Phi(\boldsymbol{\xi})$ of coordinates $\boldsymbol{\xi}$ as the convex function, possessing a specific geometric structure to fulfill the requirements for constructing the LNE divergence. For a twice differentiable function, it is considered convex if and only if its Hessian is positive-definite
\[
H(\boldsymbol{\xi})=\left(\frac{\partial^2}{\partial \xi_i \partial \xi_j} \Phi(\boldsymbol{\xi})\right).
\]

\begin{definition}
	\label{def6}
	\citep{bregman1967relaxation} The \textbf{Bregman divergence} $D_B[\boldsymbol{\xi}:\boldsymbol{\xi}']$ is defined as the difference between a convex function $\Phi(\boldsymbol{\xi})$ and its tangent hyperplane $z=\Phi(\boldsymbol{\xi}')+(\boldsymbol{\xi}-\boldsymbol{\xi}') \cdot \nabla\Phi(\boldsymbol{\xi}')$, depending on the Taylor expansion at the point $\boldsymbol{\xi}'$:
	\[
	D_B[\boldsymbol{\xi}:\boldsymbol{\xi}']=\Phi(\boldsymbol{\xi})-\Phi(\boldsymbol{\xi}')-(\boldsymbol{\xi}-\boldsymbol{\xi}') \cdot \nabla\Phi(\boldsymbol{\xi}').
	\]
\end{definition}

By drawing a tangent hyperplane that touches the convex function at the point $\boldsymbol{\xi}'$
\[
z=\Phi(\boldsymbol{\xi}')+(\boldsymbol{\xi}-\boldsymbol{\xi}') \cdot\nabla \Phi(\boldsymbol{\xi}'),
\]
we can express the distance between the convex function $\Phi(\boldsymbol{\xi})$ and the tangent hyperplane $z$ as the Bregman divergence. Since $\Phi(\boldsymbol{\xi})$ is convex, the graph of $\Phi(\boldsymbol{\xi})$ is always above the tangent hyperplane, touching it at $\boldsymbol{\xi}'$. The relationship between $\Phi(\boldsymbol{\xi})$ and $z$ is illustrated in Figure~\ref{convex}, with $z$ representing the vertical axis of the graph.

\begin{remark}
We show examples of Bregman divergence~\citep{amari2016information}. For a convex function $\Phi(\boldsymbol{\xi})=1/2 \sum_i \xi^2_i$ in a Euclidean space, the Bregman divergence coincides with the Euclidean divergence, equivalently, the square of the Euclidean distance. When considering a convex function $\Phi(\boldsymbol{\xi})=-\sum_i \log \xi_i$, the Bregman divergence is equivalent to the Logarithmic divergence. For another convex function $\Phi(\boldsymbol{\xi})= \sum_i \xi_i \log \xi_i$ satisfying $\sum_i \xi_i = 1$, the Bregman divergence is the same as the KL divergence.
% In fact, a divergence is derived from a convex function in the form of the Bregman divergence. By choosing different convex functions, for example, we can easily obtain the Euclidean divergence from the Bregman divergence, by defining the convex function: $\Phi(\boldsymbol{\xi})=1/2 \sum_i \xi^2_i$ in a Euclidean space.
% Besides, given the convex function: $\Phi(\boldsymbol{\xi})=\sum_i p(\boldsymbol{x},\xi_i) \log p(\boldsymbol{x},\xi_i)$ where $\sum_i p(\boldsymbol{x},\xi_i)=\sum_i p(\boldsymbol{x},\xi'_i)=\int p(\boldsymbol{x},\boldsymbol{\xi}) d\boldsymbol{\xi}=\int p(\boldsymbol{x},\boldsymbol{\xi}')d\boldsymbol{\xi}'=1$ is satisfied, the Bregman divergence is the same as the KL divergence.
\end{remark}

\subsubsection{LNE Divergence and Gradient}

Similar to the Bregman divergence associated with a convex function, we aim to construct a new convex function to derive the LNE divergence, from which the LNE metric naturally emerges based on Definition~\ref{def5}.
Inspired by the work of~\citet{ajanthan2021mirror}, we propose a novel convex function that satisfies symmetry and allows the geometric construction of an easy-to-compute metric with the linearly nearly Euclidean nature:
\begin{equation}
	\Phi(\boldsymbol{\xi})=\sum_i \frac{1}{\tau^2}\log \left(\cosh(\tau\xi_i)\right)
	\label{convex1}
\end{equation}
where $\tau$ is a constant parameter controlling the linearity closeness to Euclidean structure. 

\begin{theorem}
	\label{thm7}
	By introducing a convex function $\Phi$ defined by Equation~(\ref{convex1}) into Definition~\ref{def6}, the \textbf{LNE divergence} between two points $\boldsymbol{\xi}$ and $\boldsymbol{\xi}'$ can be expressed as:
	\begin{equation}
	\begin{aligned}
	D_{LNE}[\boldsymbol{\xi}':\boldsymbol{\xi}] &=\sum_i \left[\frac{1}{\tau^2}\log\frac{\cosh(\tau\xi'_i)}{\cosh(\tau\xi_i)} - \frac{1}{\tau}(\xi'_i-\xi_i)\tanh(\tau\xi_i)\right] \\
	&\approx \frac{1}{2}\sum_{i,j} \left[\delta_{ij} -\left(\tanh(\tau\boldsymbol{\xi})\tanh(\tau\boldsymbol{\xi})^\top\right)_{ij} d\xi_i d\xi_j\right]	.
	\end{aligned}
	\end{equation}
\end{theorem}
\begin{proof}
	The detailed proofs can be found in Appendix~\ref{app31}.
\end{proof}

Combined with Definition~\ref{def5}, it is evident that the \textbf{LNE metric} corresponding to the LNE divergence is given by
\begin{equation}
	\label{lne}
	\begin{aligned}
		&g(\boldsymbol{\xi}) = \delta_{ij} - \left[\tanh(\tau\boldsymbol{\xi})\tanh(\tau\boldsymbol{\xi})^\top\right]_{ij} \\
		&=\begin{bmatrix} 1-\tanh (\tau \xi_1)\tanh (\tau\xi_1)& \cdots &-\tanh (\tau\xi_1)\tanh (\tau \xi_n)\\ \vdots&\ddots&\vdots \\
			-\tanh (\tau\xi_n)\tanh (\tau \xi_1)&\cdots&1-\tanh (\tau \xi_n)\tanh (\tau\xi_n)
		\end{bmatrix}.
	\end{aligned}
\end{equation}

Building upon the concepts introduced in Section~\ref{3.1}, we can leverage the parameters of a neural network to construct the LNE metric (with the neural network's parameter vector $\boldsymbol{\xi}$). Consequently, the neural network can be characterized within the LNE manifold, as measured by the LNE divergence based on Theorem~\ref{thm7}. The steepest descent gradient in the LNE manifold is given by Lemma~\ref{lem5}, resembling the natural gradient defined in Definition~\ref{def7}.

\begin{lemma}
	\label{lem5}
	The steepest descent gradient measured by the LNE divergence is defined as
	\begin{equation}
		\tilde{\partial}_{\boldsymbol{\xi}} =g(\boldsymbol{\xi})^{-1} \partial_{\boldsymbol{\xi}} =\left[\delta - \tanh(\tau\boldsymbol{\xi})\tanh(\tau\boldsymbol{\xi})^\top\right]^{-1} \partial_{\boldsymbol{\xi}}.
	\end{equation}
\end{lemma}
\begin{proof}
	The proofs can be found in Appendix~\ref{app32}.
\end{proof}

Within the constructed LNE manifold, the introduction of the Ricci flow facilitates the decay of metric perturbations w.r.t. the LNE metric, which will be elaborated on in the following section.

\section{Evolution of LNE Manifolds under Ricci Flow}
\label{4}

This section focuses on LNE metrics under Ricci flow, aiming to demonstrate that the evolution of LNE manifolds exhibits strong stability properties over time. Specifically, we prove that the Ricci flow exponentially decays the $L^2$-norm perturbation to the LNE metric.

\subsection{LNE Metrics and Ricci Flow}
\label{4.1}

To facilitate the handling of metric perturbations, we have presented the LNE metric $g(\boldsymbol{\xi})$ in Equation~(\ref{lne}), which takes the form of Ricci-flat metrics~\citep{guenther2002stability,deruelle2021stability}. Furthermore, the definition of the LNE metric corresponds to the linearly nearly Euclidean Ricci-flat metric as per Definition~\ref{def2}, building upon prior work of~\citet{deruelle2021stability}.
Notably, the equivalence of the LNE metric $g(\boldsymbol{\xi})$ extends to either $g_0$ under the Ricci flow or $\bar{g}_0$ under the Ricci-DeTurck flow, as they are diffeomorphic\footnote{When a Ricci flow exists, a corresponding Ricci-DeTurck flow exists, and vice versa.} to each other based on Equation~(\ref{pb}).

\begin{definition}
	\label{def2}
	A complete Riemannian $n$-manifold $(\mathcal{M}^n, \bar{g}_0)$ is said to be LNE with one end of order $\iota > 0$ if there exists a compact set $K \subset \mathcal{M}$, a radius $r$, a point $x$ in $\mathcal{M}$ and a diffeomorphism satisfying $\phi : \mathcal{M} \backslash K \rightarrow (\mathbb{R}^n \backslash B(x,r))/SO(n)$. Note that $B(x,r)$ is the ball and $SO(n)$ is a finite group acting freely on $\mathbb{R}^n \backslash \{0\}$. Then
	\begin{equation}
	\left|\partial^k(\phi_* \gamma)\right|_{\delta}=O(r^{-\iota -k}) \;\;\;  \forall k\geq0
	\end{equation}
	holds on $(\mathbb{R}^n \backslash B(x,r))/SO(n)$. The LNE metric $\bar{g}_0$ can be linearly decomposed into a form containing the Euclidean metric $\delta$ and the deviation $\gamma$:
	\begin{equation}
	\bar{g}_0=\delta+\gamma.
  \label{lnemetric}
	\end{equation}
\end{definition}

\subsection{All Time Convergence for $L^2$-norm Perturbations}
\label{4.3}

Firstly, buliding upon previous proofs~\citep{koiso1983einstein,besse2007einstein}, we can establish that the LNE manifold $(\mathcal{M}^n, g_0)$ is integral and linearly stable, as defined in Definition~\ref{def3} and Definition~\ref{def4}.

\begin{definition}
	\label{def3}
	\citep{deruelle2021stability} A complete LNE $n$-manifold $(\mathcal{M}^n, g_0)$ is said to be linearly stable if the $L^2$ spectrum of the Lichnerowicz operator $L_{g_0}:=\Delta_{g_0}+2\operatorname{Rm}(g_0)*$ is in $(-\infty,0]$ where $\Delta_{g_0}$ is the Laplacian, when $L_{g_0}$ acting on $d_{ij}$ satisfies
	\begin{equation}
	\begin{aligned}
	L_{g_0}(d)&=\Delta_{g_0}d+2\operatorname{Rm}(g_0)*d \\
	&=\Delta_{g_0}d+2\operatorname{Rm}(g_0)_{iklj}d_{mn}g_0^{km}g_0^{ln}.
	\end{aligned}
  \end{equation}
\end{definition}

\begin{definition}
	\label{def4}
	\citep{deruelle2021stability} A $n$-manifold $(\mathcal{M}^n, g_0)$ is said to be integrable if a neighbourhood of $g_0$ has a smooth structure.
\end{definition}

Due to the diffeomorphic relationship between the Ricci flow and the Ricci–DeTurck flow, we introduce a metric perturbation for the Ricci-DeTurck flow, and Equation~(\ref{deturck}) can be further reformulated as follows:
\begin{equation}
	\begin{aligned}
		\frac{\partial}{\partial t}\bar{g}(t)&=-2 \operatorname{Ric}(\bar{g}(t))- \mathcal{L}_{\frac{\partial \varphi(t)}{\partial t}} \bar{g}(t) \\
		\bar{g}(0) &=\bar{g}_0 + d
	\end{aligned}
	\label{newdeturck}
\end{equation}
where $d=\bar{g}(0)-\bar{g}_0$ is a metric perturbation deviated from the LNE metric $\bar{g}_0$.
In this way, $d(t) - d_0(t)=\bar{g}(t)-\bar{g}_0(t)$ holds because we define $d_0(t)=\bar{g}_0(t)-\bar{g}_0$.

\begin{theorem}
	\label{thm6}
	Let $(\mathcal{M}^n, \bar{g}_0)$ be the LNE $n$-manifold which is linearly stable and integrable. For any metric $\bar{g}(t) \in \mathcal{B}_{L^2}(\bar{g}_0, \epsilon_2)$ where a constant $\epsilon_2 > 0$, there is a complete Ricci–DeTurck flow $(\mathcal{M}^n, \bar{g}(t))$ starting from $\bar{g}(t)$ converging to the LNE metric
	$\bar{g}(\infty) \in \mathcal{B}_{L^2}(\bar{g}_0, \epsilon_1)$ where $\epsilon_1$ is a small enough constant.
\end{theorem}
\begin{proof}
The proofs can be found in Appendix~\ref{key}.
\end{proof}

According to Theorem~\ref{thm6}, the $L^2$-norm metric perturbation w.r.t. the LNE metric can be dynamically decayed by the Ricci-DeTurck flow in all time. For more details, refer to Appendix~\ref{applong} (finite-time stability in Appendix~\ref{fts} and all-time stability in Appendix~\ref{ats}). By proving the finite time existence of the Ricci-DeTurck flow with $L^2$-norm perturbations (Corollary~\ref{cor2}), we then establish the convergence of $L^2$-norm perturbations w.r.t. the LNE metric for all time under the Ricci-DeTurck flow (Theorem~\ref{thm6}).

\subsection{Perturbation Analysis}
\label{4.4}

Following the analysis in~\citep{sesum2006linear}, we further obtain $|\bar{g}(t)-\bar{g}_0(\infty)|<C e^{-\epsilon_2 t}$, indicating exponential convergence of the metric perturbation. Consequently, it also exhibits exponential convergence for $g(t)$ under the Ricci flow, assuming the existence of a solution of the Ricci flow. Recall that by reparameterizing $\bar{g}(t)$ to $g(t)=\varphi^*(t) \bar{g}(t)$ via the pullback, the perturbation entirely originates from $\bar{g}(t)$ and is independent of the time-dependent diffeomorphism $\varphi^*(t)$. 

In Section~\ref{3}, the metric $g(\boldsymbol{\xi}) = \delta_{ij} - \left[\tanh(\tau\boldsymbol{\xi})\tanh(\tau\boldsymbol{\xi})^\top\right]_{ij}$ constructed for the neural network is a kind of LNE metrics (as per Definition~\ref{def2}), thereby ensuring the perturbation for this metric undergoes exponential decay under the Ricci flow.

\section{Discretized Neural Networks in LNE Manifolds}
\label{5}

Up to this point, we have tackled the problem of gradient mismatch by constructing LNE manifolds for neural networks (Section~\ref{3}) and implementing an exponential decay mechanism for metric perturbations (Section~\ref{4}). 
However, the practical computation of the steepest descent gradient in the LNE manifold, as indicated by Lemma~\ref{lem5}, poses challenges due to the involovement of the inverse of the LNE metric. In this section, our objective is to approximate the inverse of the LNE metric and sebsequently derive the approximated gradient in the LNE manifold. This step is crucial for developing a practical algorithm to train DNNs in the LNE manifold.

\subsection{Gradient Computation in Discretized Neural Networks}
\label{5.1}

Recall that \citet{courbariaux2016binarized} applied STE to binarized neural networks, formulated as in Equation~(\ref{bste}). Subsequently, \citet{zhou2016dorefa} extended STE to arbitrary bit-width discretized neural networks. The generalized form of STE in discretized neural networks is expressed as:
\begin{equation}
	\frac{\partial L}{\partial \boldsymbol{w}}=\frac{\partial L}{\partial Q(\boldsymbol{w})}.
	\label{ste}
\end{equation}

Before introducing the LNE manifold to DNNs, a contradiction needs resolution. According to Lemma~\ref{lem5}, the LNE manifold is defined based on the parameter $\boldsymbol{\xi}$ across all layers in a neural network. However, back-propagation computes the gradient layer-by-layer, specifically on the weight $\boldsymbol{w}$ of each layer. This misalignment prevents the direct association of gradient updates with the LNE manifold. Fortunately, we can redefine the LNE manifold layer-by-layer by substituting $\boldsymbol{\xi}$ with $\boldsymbol{w}$, effectively defining the LNE manifold for each layer. Building upon Lemma~\ref{lem5}, the steepest descent gradient is then reformulated as:
\begin{equation}
	\tilde{\partial}_{\boldsymbol{w}} =g^{-1}(\boldsymbol{w}) \partial_{\boldsymbol{w}} =\left[\delta - \tanh(\tau\boldsymbol{w})\tanh(\tau\boldsymbol{w})^\top\right]^{-1} \partial_{\boldsymbol{w}},
\end{equation}
which can be used for the gradient computation in DNNs, i.e.,
\begin{equation}
	\label{match}
	\frac{\tilde{\partial} L}{\tilde{\partial} \boldsymbol{w}}=\left[\delta - \tanh(\tau\boldsymbol{w})\tanh(\tau\boldsymbol{w})^\top\right]^{-1}\frac{\partial L}{\partial Q(\boldsymbol{w})}.
\end{equation}
Furthermore, the proposed gradient, as described above, is part of our framework to address the problem of gradient mismatch, based on Equation~(\ref{metripertur}). In this context, the metric is layer-by-layer LNE.
However, the computation of the gradient involves the inverse of the LNE metric, a process that demands significant computational resources. Hence, we introduce two methods for approximating the gradient of DNNs in LNE manifolds: weak approximation and strong approximation, respectively.
The approximated gradient is defined as the direction in parameter space that maximizes the objective's variation per unit change along the layer-by-layer LNE manifold. 

\subsection{Strong Approximation}
\label{5.2}

Our objective is to approximate the inverse of the LNE metric and subsequently approximate the gradient in Equation~(\ref{match}).
Based on the universal approximation theorem~\citep{cybenko1989approximation,hornik1991approximation}, which asserts that a continuous function on compact subsets can be approximated by a neural network with a single hidden layer and a finite number of neurons~\citep{jejjala2020neural}, we introduce a Multi-Layer Perceptron (MLP) neural network, depicted in Figure~\ref{approx}, to minimize the loss function:
\begin{equation}
	\tilde{L}=\|\boldsymbol{I}-g(\boldsymbol{w})\boldsymbol{G}\|^2.
	\label{loss}
\end{equation}
For the $n \times n$ symmetric metric $g(\boldsymbol{w})$, it can be decomposed into the combination of entries $P$ and $A$, where $P$ consists of the elements of the lower triangular matrix, containing $n(n-1)/2$ real parameters, and $A$ consists of the elements of the diagonal matrix, containing $n$ real parameters. 
Therefore, the matrix $\boldsymbol{G}$ can be effectively utilized to strongly approximate the inverse of the metric $g(\boldsymbol{w})$. 

\begin{figure}[th]
	\centering
	\includegraphics[width=\textwidth]{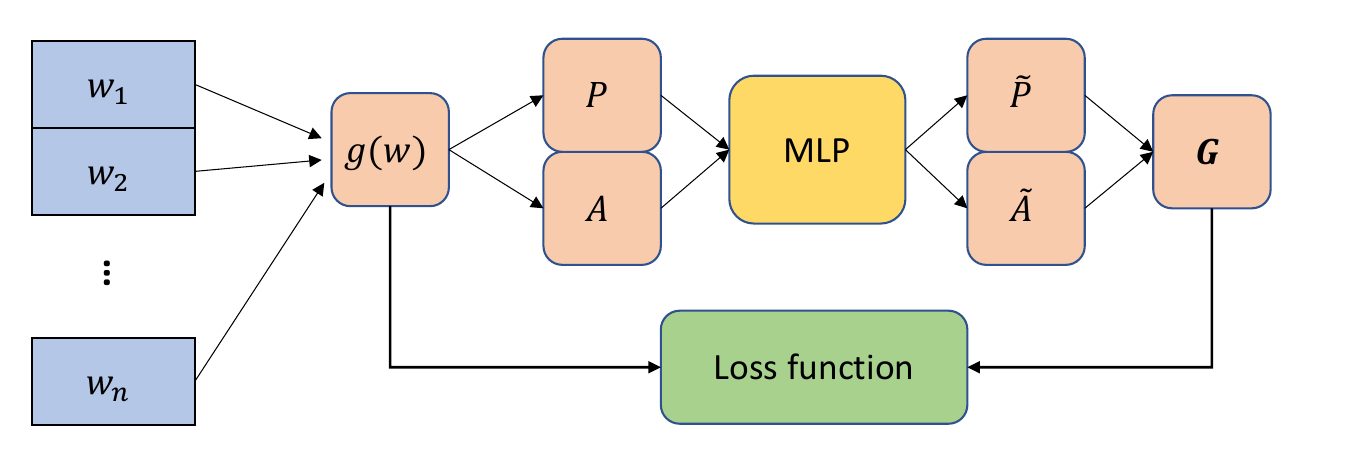}
	\caption{The flow chart of strong approximation of $g^{-1}(\boldsymbol{w})$. The new entries $\tilde{P}$ and $\tilde{A}$ generated by the neural network constitute a matrix $\boldsymbol{G}$, which is multiplied by the metric $g(\boldsymbol{w})$. As the loss function, defined by Equation~(\ref{loss}), decreases, the matrix $\boldsymbol{G}$ serves to approximate the inverse of the metric $g(\boldsymbol{w})$.}
	\label{approx}
\end{figure}

\subsection{Weak Approximation}
\label{5.3}

In this subsection, we present a method for the weak approximation of the inverse of the LNE metric with efficient calculations.

\begin{definition}
	\label{strict}
	For $\boldsymbol{A} \in \mathcal{R}^{n \times n}$, $\boldsymbol{A}$ is called \textbf{diagonally dominant} when it satisfies
	\[
	\big|a_{ii}\big| > \sum^n_{j=1,j\neq i} \big|a_{ij}\big|,\;\;\; i=1,2,\ldots,n.
	\]
\end{definition}

\begin{definition}
	\label{nonsingular}
	If $\boldsymbol{A} \in \mathcal{R}^{n \times n}$ is a diagonally dominant matrix, then $\boldsymbol{A}$ is a \textbf{nonsingular matrix} together, i.e., $\boldsymbol{A}^{-1}$ exists.
\end{definition}

By considering the properties of the LNE metric, adjusting the parameter $\tau$ allows us to easily ensure that the LNE metric $g(\boldsymbol{w})$ is diagonally dominant based on Definition~\ref{strict}. Moreover, the existence of $g^{-1}(\boldsymbol{w})$ can be guaranteed based on Definition~\ref{nonsingular}. According to Corollary~\ref{cor4}, the weak approximation of the gradient in the LNE manifold can be calculated, offering a convenient feature for accelerating the computation of the inverse.

\begin{corollary}
	\label{cor4}
	Based on Definition~\ref{strict} and Definition~\ref{nonsingular}, the weak approximation of the gradient in the LNE manifold is defined as
	\begin{equation}
	\tilde{\partial}_{\boldsymbol{w}}=\left[\delta - \tanh(\tau\boldsymbol{w})\tanh(\tau\boldsymbol{w})^\top\right]^{-1} \partial_{\boldsymbol{w}} \approx \left[\delta + \tanh(\tau\boldsymbol{w})\tanh(\tau\boldsymbol{w})^\top\right] \partial_{\boldsymbol{w}}
	\end{equation}
	if the LNE metric is diagonally dominant.
\end{corollary}
\begin{proof}
	Considering the inverse of the LNE metric, due to the diagonally dominant property in Definition~\ref{strict} and Definition~\ref{nonsingular}, we can approximate $\left[\delta - \tanh(\tau\boldsymbol{w})\tanh(\tau\boldsymbol{w})^\top\right]^{-1}$ by ignoring the fourth-order small quantity $\sum O(\rho_a \rho_b \rho_c \rho_d)$, i.e.,
	\[
	\begin{aligned}
		&\left[\delta - \tanh(\tau\boldsymbol{w})\tanh(\tau\boldsymbol{w})^\top\right]\left[\delta + \tanh(\tau\boldsymbol{w})\tanh(\tau\boldsymbol{w})^\top\right] \\
		&= \begin{bmatrix} 1-\rho_1 \rho_1 & -\rho_1 \rho_2 &\cdots\\ -\rho_2 \rho_1 & 1-\rho_2 \rho_2 &\cdots \\
			\vdots&\vdots&\ddots
		\end{bmatrix} \begin{bmatrix} 1+\rho_1 \rho_1 & \rho_1 \rho_2 &\cdots\\ \rho_2 \rho_1 & 1+\rho_2 \rho_2 &\cdots \\
			\vdots&\vdots&\ddots
		\end{bmatrix} \\
		& = \begin{bmatrix} 1-\sum O(\rho_a \rho_b \rho_c \rho_d)& \rho_1 \rho_2-\rho_1 \rho_2 -\sum O(\rho_a \rho_b \rho_c \rho_d) &\cdots\\ -\rho_2 \rho_1 + \rho_2 \rho_1 -\sum O(\rho_a \rho_b \rho_c \rho_d) & 1-\sum O(\rho_a \rho_b \rho_c \rho_d) &\cdots \\
			\vdots&\vdots&\ddots
		\end{bmatrix} \approx \boldsymbol{I}.
	\end{aligned}
	\]
The proof is completed.
\end{proof}
%By now, we introduce the diagonally dominant property such that the weak approximation of the inverse of the LNE metric is presented.

\subsection{Training}
\label{5.4}

\begin{algorithm}[htb]
	\caption{An algorithm for training DNNs in the LNE manifold. We denote the gradient in the LNE manifold as $\tilde{\partial}$. For brevity, we omit the normalization operation~\citep{ioffe2015batch,ba2016layer}.}
	\label{alg1}
	\begin{algorithmic}[1]
		\Require
		A minibatch of inputs and targets $(\boldsymbol{x}=\boldsymbol{a_0},\boldsymbol{y})$, $\boldsymbol{\xi}$ mapped to $\left(\boldsymbol{W}_{1}, \boldsymbol{W}_{2}, \ldots, \boldsymbol{W}_{l}\right)$, $\boldsymbol{\hat{\xi}}$ mapped to $\left(\boldsymbol{\hat{W}}_{1}, \boldsymbol{\hat{W}}_{2}, \ldots, \boldsymbol{\hat{W}}_{l}\right)$, a nonlinear function $f$, a constant factor $\tau$ and a learning rate $\eta$.
		\Ensure
		The updated discretized parameters $\boldsymbol{\hat{\xi}}$.
		\State \{Forward propagation\}
		\For{$i=1;i \le l;i++$}
		\State Discretize $\boldsymbol{\hat{W}}_{i}=Q(\boldsymbol{W}_i)$;
		\State Compute $\boldsymbol{s}_{i}= \boldsymbol{\hat{W}}_{i}\boldsymbol{\hat{a}}_{i-1}$;
		\State Discretize $\boldsymbol{\hat{a}}_{i}=Q\left(f(\boldsymbol{s}_{i}\right))$;
		\EndFor 
		\State \{Loss derivative\}
		\State Compute $L=L(\boldsymbol{y},\boldsymbol{z})$;
		\State Compute $\partial_{\boldsymbol{a}_l} L=\frac{\partial L(\boldsymbol{y},\boldsymbol{z})}{\partial \boldsymbol{z}}\big|_{\boldsymbol{z}=\boldsymbol{\hat{a}}_l}$;
		\State \{Backward propagation\}
		\For{$i=l;i \ge 1;i--$}
		\State Compute $\partial_{\boldsymbol{s}_i} L=\partial_{\boldsymbol{a}_i} L \odot f'(\boldsymbol{s}_i)$;
		\State Compute $\partial_{\boldsymbol{\hat{W}}_i} L=\left(\nabla_{\boldsymbol{s}_i} L\right) \boldsymbol{\hat{a}}_{i-1}^\top$;
		\State Compute $\tilde{\partial}_{\boldsymbol{W}_i} L=g^{-1}(\boldsymbol{W}_i)\partial_{\boldsymbol{\hat{W}}_i} L$ based on Equation~(\ref{match});
		\State Compute $\partial_{\boldsymbol{\hat{a}}_{i-1}} L=  \boldsymbol{\hat{W}}_{i}^\top \left(\partial_{\boldsymbol{s}_i} L\right)$;
		\EndFor
		\State \{The parameters update\}
		\For{$i=l;i \ge 1;i--$}
		\State Update $\boldsymbol{W}_i \leftarrow \boldsymbol{W}_i-\eta\cdot\tilde{\partial}_{\boldsymbol{W}_i} L$;
		\EndFor
		\State Update $\boldsymbol{\hat{\xi}}=\left[\operatorname{vec}\left(\boldsymbol{\hat{W}}_{1}\right)^{\top}, \operatorname{vec}\left(\boldsymbol{\hat{W}}_{2}\right)^{\top}, \ldots, \operatorname{vec}\left(\boldsymbol{\hat{W}}_{l}\right)^{\top}\right]^{\top}$;
	\end{algorithmic}
\end{algorithm}

Building upon previous work~\citep{courbariaux2016binarized}, we present a practical algorithm for training DNNs in the LNE manifold. As outlined in Algorithm~\ref{alg1}, this algorithm closely resembles the general DNN training algorithm, with the key difference lying in Line 14. Recall that, in Figure~\ref{intro}, conventional DNNs utilize STE to directly copy the gradient, i.e., $\tilde{\partial}_{\boldsymbol{W}_i} L=\partial_{\boldsymbol{\hat{W}}_i} L$. In contrast, our method involves matching the gradient by introducing the LNE metric.
Moreover, we can practically compute this gradient in Line 14 using either the strong approximation or weak approximation mentioned above.

\section{Ricci Flow Discretized Neural Networks}
\label{6}

In this section, we introduce Ricci flow discretized neural networks (RF-DNNs). The introduction of the Ricci flow implies that the background of the discussed DNNs is the LNE manifold. Our primary goal is to offer a practical solution for metric perturbations, thereby addressing the problem of gradient mismatch. Thus, we will focus on the practical calculation of discrete Ricci flow, rather than solely engaging in theoretical analysis.

To establish the connection between the Ricci flow and neural networks, we discretize the Ricci flow and select a suitable coordinate system. In Section~\ref{3}, we have established the relationship between the LNE metric and neural networks for the left-hand side of the Ricci flow. Notably, we utilize the form of the LNE metric in these calculations, and such metrics at this stage incorporate perturbations. Moving to the right-hand side of the Ricci flow, we need to compute the Ricci curvature tensor with the chosen coordinate system. This coordinate system is important for linking Ricci curvature to neural networks. Specifically, we define a method for calculating the Ricci curvature, where the selection of coordinate systems is associated with input transformations. This implies that the Ricci curvature in neural networks reflects the impact of different input transformations on the parameters.

\subsection{Ricci Curvature in Neural Networks}
\label{6.1}

Now, let's consider the Ricci curvature tensor on the Riemannian metric $g$. According to Appendix~\ref{app1}, its coordinate form can be expressed as follows:
\begin{equation}
\begin{aligned}
&-2\operatorname{Ric}(g)=-2R_{i k j}^{i}=2R_{k i j}^{i} \\
&=g^{i p}\left(\partial_{i}\partial_{k}g_{p j}-\partial_{i}\partial_{j}g_{p k}+\partial_{p}\partial_{j}g_{i k}+\partial_{p}\partial_{k}g_{i j}\right).
\end{aligned}
\end{equation}
To establish a connection between the Ricci curvature and neural networks, we introduce a method for calculating the Ricci curvature such that the selection of coordinate systems is linked to input transformations. 
When the Ricci curvature is equal to zero, it implies that different input transformations will not induce variations in the parameters.

Inspired by prior work~\citep{kaul2019riemannian}, we interpret the terms $\partial_{i}$ and $\partial_{p}$ as changes representing translation and rotation of each input, respectively.
Typically, data augmentation in real-world applications like image classification tasks~\citep{he2016deep,shorten2019survey} does not involve rotation. For the sake of fairness in ablation studies, we focus on translation by discarding the index $p$, i.e., $\partial_{p}(\partial_{j}g_{i k}+\partial_{k}g_{i j})=0$. When considering either translation or rotation, $g^{ip}$ degenerates into $\delta^{ip}$ (the identity matrix).
Consequently, $\partial_{i}\partial_{k}g$ and $\partial_{i}\partial_{j}g$ can be treated as changes representing row and column transformations of the input data w.r.t. the metric $g$, respectively. The Ricci curvature can be rewritten as:
\begin{equation}
\label{riccicur}
-2\operatorname{Ric}(g)
=\partial_{i}\partial_{k}g_{pj}-\partial_{i}\partial_{j}g_{pk}.
\end{equation}

\begin{remark} 
The selection of $i$ and $p$ (as well as $k$ and $j$) is arbitrary and can even be represented in other coordinate systems. Here, we provide a specific geometric meaning by considering the characteristic of the image classification task.
\end{remark}

\begin{figure}[thbp]
	\centering
  \subfigure[Metric structure ($g$)]
	{\includegraphics[width=.49\textwidth]{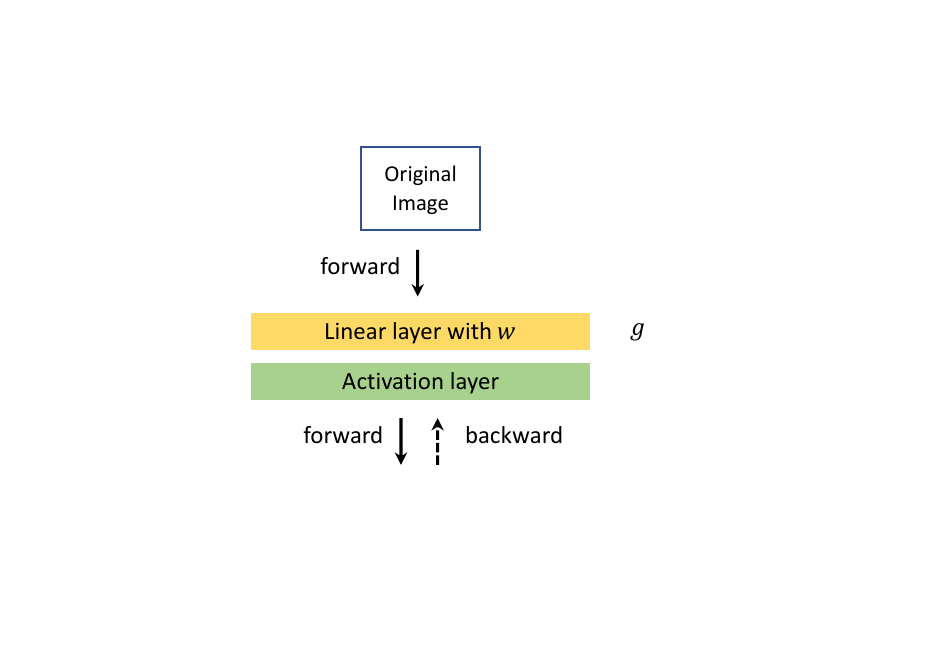}}
  \subfigure[Ricci curvature structure ($\operatorname{Ric}(g)$)]
	{\includegraphics[width=.49\textwidth]{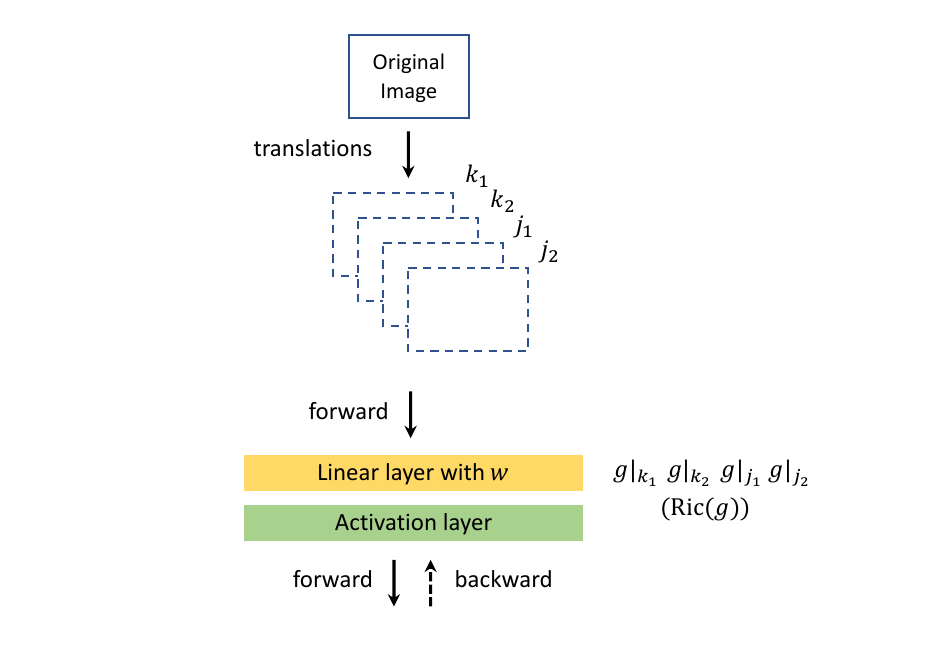}}
	\caption{Upon feeding the original image into the neural network and performing a forward and backward pass on the linear layer to update the weights $\boldsymbol{w}$, we construct the metric structure $g(\boldsymbol{w})$ based on Section~\ref{5.1}. Furthermore, we subject the original image to four distinct small translation transformations ($k_1$, $k_2$, $j_1$, and $j_2$) before inputting them into the neural network. By sequentially performing a forward and backward passes, we obtain four metric structures ($g|_{k_1}$, $g|_{k_2}$, $g|_{j_1}$, and $g|_{j_2}$) corresponding to these translations. The combination of these metrics allows us to characterize the Ricci curvature $\operatorname{Ric}(g)$.}
	\label{graph}
\end{figure}

As shown in Figure~\ref{graph} and leveraging Equation~(\ref{riccicur}), we express the Ricci curvature with coordinate systems using a difference equation:
\begin{equation}
  \label{condition}
  -2\operatorname{Ric}(g)
  =\frac{g|_{k_1} - g|_{k_2}}{k_1-k_2}-\frac{g|_{j_1} - g|_{j_2}}{j_1-j_2}
\end{equation}
where we approximate partial derivatives with difference equations~\citep{kaul2019riemannian}, i.e., $\partial_{i}\partial_{k}g=(g|_{k_1} - g|_{k_2})/(k_1-k_2)$ and $\partial_{i}\partial_{j}g=(g|_{j_1} - g|_{j_2})/(j_1-j_2)$ corresponding to the input translation dimensions $k$ and $j$, respectively.
Here, $g|_{k_1}$, $g|_{k_2}$, $g|_{j_1}$, and $g|_{j_2}$ represent four metric structures under different small translation transformations $k_1$, $k_2$, $j_1$, and $j_2$, respectively.
In general, $(k_1-k_2)$ and $(j_1-j_2)$ denote translations of fewer than 4 pixels, aligning with common data augmentation practices~\citep{he2016deep}.

\subsection{Existence of Discrete Ricci Flow in Neural Networks}
\label{6.2}

Recall that we considered the Ricci-DeTurck flow instead of the Ricci flow, as the solution of the Ricci flow does not always exist, as discussed in Section~\ref{2.3}.
Assuming that the solution of the Ricci flow exists in neural networks, we can utilize the Ricci flow to exponentially decay the metric perturbation, as explained in Section~\ref{4.4}.

In terms of the Ricci flow equation, we have previously examined the right-hand side, namely, the Ricci curvature tensor. Now, we define the equivalent form of the left-hand side of the Ricci flow using a difference equation:
\begin{equation}
    \frac{\partial}{\partial t} g(t) := g(t+1)-g(t),
    \label{left}
\end{equation}
which represents the consecutive iterations in the training process, where $t\in \{0,1,\cdots,T-1\}$ is a uniform partition of the interval $[0,T]$, with $T$ being the total number of iterations. As the number of iterations $T$ approaches infinity, the formula above holds.

Combining Equation~(\ref{condition}) and Equation~(\ref{left}) in neural networks, we present the discrete Ricci flow as a difference equation:
\begin{equation}
\begin{aligned}
    g(t+1)|_{k_1}-g(t)|_{k_1}&=\frac{g(t)|_{k_1} - g(t)|_{k_2}}{k_1-k_2} - \frac{g(t)|_{j_1} - g(t)|_{j_2}}{j_1-j_2} \\
    g(0)|_{k_1} &= \delta - \tanh(\tau\boldsymbol{w})\tanh(\tau\boldsymbol{w})^\top
\end{aligned}
\label{DRF}
\end{equation}
To ensure the existence of the solution of the discrete Ricci flow, we achieve this goal by adding a regularization term to the loss function, constraining the discrete Ricci flow in DNNs. Following Equation~(\ref{DRF}), we present the regularization term:
\begin{equation}
\label{norm}
%\sideset{^{g(t)}}{}
{\mathop{N}}=\left\|g(t+1)|_{k_1}-g(t)|_{k_1}-\frac{g(t)|_{k_1} - g(t)|_{k_2}}{k_1-k_2} + \frac{g(t)|_{j_1} - g(t)|_{j_2}}{j_1-j_2}\right\|^{2}_{L^2},
\end{equation}
where $g(t)$ is $\epsilon$-close to the LNE metric $g_0$ based on Definition~\ref{eball}. In other words, $g(t)$ is the LNE metric with perturbations. 

\begin{definition}
	\label{eball}
	\citep{sheridan2006hamilton} Let $g(t)$ be the metrics on the LNE manifold. For $\epsilon > 0$, $\mathcal{B}_{L^2}(g_0, \epsilon)$ is the $\epsilon$-ball with respect to the $L^2$-norm induced by $g_0$ and centred at $g_0$, where any metric $g(t) \in \mathcal{B}_{L^2}(g_0, \epsilon)$ is $\epsilon$-close to $g_0$ if
  \[
    (1+\epsilon)^{-1} g_0 \leq g(t) \leq (1+\epsilon) g_0
  \]
  in the sense of matrices.
\end{definition}

By constraining the regularization term $N$ in DNNs, the solution of the discrete Ricci flow exists when $N \rightarrow 0$. Simultaneously, the metric perturbation exponentially converges ($g(t) \rightarrow g_0$) as the discrete Ricci flow evolves.

\subsection{Algorithm Design}
\label{6.3}

\begin{algorithm}[htb]
	\caption{An algorithm for training our RF-DNNs in the LNE manifold. We introduce a parameter $\alpha$ to balance the regularization and ensure the existence of the solution for the discrete Ricci flow. For brevity, we omit the normalization operation~\citep{ioffe2015batch,ba2016layer}.}
	\label{alg2}
	\begin{algorithmic}[1]
		\Require
		A minibatch of inputs and targets $(\boldsymbol{x}=\boldsymbol{a_0},\boldsymbol{y})$, $\boldsymbol{\xi}$ mapped to $\left(\boldsymbol{W}_{1}, \boldsymbol{W}_{2}, \ldots, \boldsymbol{W}_{l}\right)$, $\boldsymbol{\hat{\xi}}$ mapped to $\left(\boldsymbol{\hat{W}}_{1}, \boldsymbol{\hat{W}}_{2}, \ldots, \boldsymbol{\hat{W}}_{l}\right)$, a nonlinear function $f$, a constant factor $\tau$ and a learning rate $\eta$.
		\Ensure
		The updated discretized parameters $\boldsymbol{\hat{\xi}}$.
		\State \{Forward propagation\}
		\For{$i=1;i \le l;i++$}
		\State Compute $\boldsymbol{\hat{W}}_{i}=Q(\boldsymbol{W}_i)$;
		\State Compute $\boldsymbol{s}_{i}= \boldsymbol{\hat{W}}_{i}\boldsymbol{\hat{a}}_{i-1}$;
		\State Compute $\boldsymbol{\hat{a}}_{i}=Q\left(f(\boldsymbol{s}_{i}\right))$;
		\EndFor
		\State Compute the regularization term $N$ based on Equation~(\ref{norm});
		\State \{Loss derivative\}
		\State Compute $L=L(\boldsymbol{y},\boldsymbol{z})+\alpha \cdot N$;
		\State Compute $\partial_{\boldsymbol{a}_l} L=\frac{\partial L(\boldsymbol{y},\boldsymbol{z})}{\partial \boldsymbol{z}}\big|_{\boldsymbol{z}=\boldsymbol{\hat{a}}_l}$;
		\State \{Backward propagation\}
		\For{$i=l;i \ge 1;i--$}
		\State Compute $\partial_{\boldsymbol{s}_i} L=\partial_{\boldsymbol{a}_i} L \odot f'(\boldsymbol{s}_i)$;
		\State Compute $\partial_{\boldsymbol{\hat{W}}_i} L=\left(\nabla_{\boldsymbol{s}_i} L\right) \boldsymbol{\hat{a}}_{i-1}^\top$;
		\State Compute $\tilde{\partial}_{\boldsymbol{W}_i} L=g^{-1}_{\boldsymbol{W}_i}(t)\partial_{\boldsymbol{\hat{W}}_i} L$ based on Equation~(\ref{rfste});
		\State Compute $\partial_{\boldsymbol{\hat{a}}_{i-1}} L=  \boldsymbol{\hat{W}}_{i}^\top \left(\partial_{\boldsymbol{s}_i} L\right)$;
		\EndFor
		\State \{The parameters update\}
		\For{$i=l;i \ge 1;i--$}
		\State Update $\boldsymbol{W}_i \leftarrow \boldsymbol{W}_i-\eta\cdot\tilde{\partial}_{\boldsymbol{W}_i} L$;
		\EndFor
		\State Update $\boldsymbol{\hat{\xi}}=\left[\operatorname{vec}\left(\boldsymbol{\hat{W}}_{1}\right)^{\top}, \operatorname{vec}\left(\boldsymbol{\hat{W}}_{2}\right)^{\top}, \ldots, \operatorname{vec}\left(\boldsymbol{\hat{W}}_{l}\right)^{\top}\right]^{\top}$;
	\end{algorithmic}
\end{algorithm}

By imposing constraints on the discrete Ricci Flow in layer-by-layer LNE manifold, we can effectively address the problem of gradient mismatch. Given that the background is the LNE manifold, we can construct the satisfied gradient based on Equation~(\ref{match}). Note that, at this point, the metric becomes time-dependent under the Ricci flow, i.e., $g_{\boldsymbol{w}}(t)$. And we obtain the gradient under the discrete Ricci flow as follows:
\begin{equation}
  \tilde{\partial}_{\boldsymbol{w}} L = g_{\boldsymbol{w}}^{-1}(t) \partial_{Q(\boldsymbol{w})}.
\label{rfste}
\end{equation}
The overall process is shown in Algorithm~\ref{alg2}. Compared with Algorithm~\ref{alg1}, we have introduced Line 7 and Line 15. In Line 7, the regularization term is calculated to ensure the existence of the solution for the discrete Ricci flow. On the other hand, in Line 15, the gradient is computed in the LNE manifold under the discrete Ricci flow. This is in contrast to Algorithm~\ref{alg1}, which only calculates the gradient in the LNE manifold with perturbations. Applying the Ricci flow indicates that the LNE manifold at this point is dynamic and anti-perturbative.

\begin{remark}
In addition to using discretized weights and activations, DNNs need to store non-discretized weights and activations for gradient updates. It is important to note that the gradients of a DNN are non-discretized.
\end{remark}

\subsection{Complexity Analysis}
\label{6.4}

Based on Algorithm~\ref{alg2}, it is evident that the forward time complexity is approximately $\mathcal{O}(n^2)$, where the time complexities of Line 4 and Line 5 are about $\mathcal{O}(n^2)$ and $\mathcal{O}(n)$, respectively. In the backward pass, the time complexity of Line 13 is around $\mathcal{O}(n)$. Consequently, the time complexities of computing the gradients w.r.t. the weights (Line 15) and activations (Line 17) are both approximately $\mathcal{O}(n^2)$. Therefore, the backward time complexity is roughly $\mathcal{O}(2n^2)$. For the training process of a neural network, its total complexity is $\mathcal{O}(n^2)$.

Since the computation of the Ricci curvature involves four different translations of input data w.r.t. the metric, its time complexity is about $\mathcal{O}(n^2)$. In this manner, the updated weights are only used to calculate the constraints of the discrete Ricci flow, and the final weights can be obtained by a subsequent backward pass. The time complexity of Line 16 is $\mathcal{O}(n^2)$ when we use the weak approximation to calculate the gradient. Thus, the total complexity of RF-DNN remains $\mathcal{O}(n^2)$, which is consistent with that of a neural network.

\section{Experiments}
\label{7}

In this section, we conduct ablation studies to compare our RF-DNN\footnote{For convenient gradient calculation, we utilize the weak approximation of the inverse of the LNE metric in all experiments.} trained from scratch with other STE methods. Additionally, when evaluating the performance of the RF-DNN with a pre-trained model, we compare it with several representative training-based methods on classification benchmark datasets. All experiments are implemented in Python using PyTorch~\citep{paszke2019pytorch}. The hardware environment includes an Intel(R) Xeon(R) Silver 4214 CPU(2.20 GHz), GeForce GTX 2080Ti GPU, and 128GB RAM.

\subsection{Experimental Settings}
\label{7.1}

The two datasets used in our experiments are introduced as follows.

{\bf CIFAR datasets:} There are two CIFAR benchmarks~\citep{krizhevsky2009learning}, each consisting of natural color images with 32 $\times$ 32 pixels. Both datasets comprise 50k training images, 10k test images, and a validation set of 5k images selected from the training set. CIFAR-10 is organized into 10 classes, while CIFAR-100 has 100 classes. We apply a standard data augmentation scheme (random corner cropping and random flipping), widely used for these two datasets.
Images are normalized during preprocessing using the means and standard deviations of the channels.

{\bf ImageNet dataset:} The ImageNet benchmark~\citep{russakovsky2015imagenet} consists of 1.2 million high-resolution natural images, with a validation set containing 50k images. These images are organized into 1000 object categories for training and re resized to 224 $\times$ 224 pixels before fed into the network. In the subsequent experiments, we report our single-crop evaluation results using top-1 and top-5 accuracies.

We specify the discrete function, the composition of which significantly influences the performance and computation of DNNs. Specifically, the discrete function can simplify calculations, which vary depending on different discrete values, such as fixed-point multiplication, SHIFT operation~\citep{elhoushi2019deepshift}, and XNOR operation~\citep{rastegari2016xnor}, etc.

We denote $Q^1$ as the $1$-bit discrete function:
\begin{equation}
Q^1(\cdot) = \operatorname{sign}(\cdot)=\{-1,+1\}.
\label{eq42}
\end{equation}
The $k$-bit, for $k>1$, discrete function is denoted as $Q^k$:
\begin{equation}
Q^{k>1}(\cdot) = \frac{2}{2^k-1}\operatorname{round}\left[(2^k-1)\left(\frac{\cdot}{2\max \lvert\cdot\rvert}+\frac{1}{2}\right)\right] - 1
\label{eq43}
\end{equation}
where $\operatorname{round}[\cdot]$ is the rounding function and $\max \lvert\cdot\rvert$ refers to calculating the absolute value of the input first, and then finding its maximum value. In this way, a DNN using the discrete function $Q^1(\cdot)$ can be computed with the XNOR operation, while a DNN using the discrete function $Q^{k>1}(\cdot)$ can be computed with fixed-point multiplication.

\subsection{Ablation Studies with STE Methods}
\label{7.2}

To showcase the superiority of RF-DNN in addressing the problem of gradient mismatch, we compare it with three other methods by training from scratch. In Table~\ref{table1}, Table~\ref{table2}, and Table~\ref{table3}, we mark $\{-1,+1\}$ in `\textbf{Forward}' to indicate that the weights are binarized using Equation~(\ref{eq42}), i.e., $-1$ or $+1$, in the forward pass of DNNs. In the backward pass, the methods (Dorefa~\citep{zhou2016dorefa}, MultiFCG~\citep{NEURIPS2019f8e59f4b}, and FCGrad~\citep{NEURIPS2019f8e59f4b}) use different approximated gradients to update the weights. Here, we apply different ResNet models~\citep{he2016deep} for ablation studies.

\begin{table}[t]
	\caption{The experimental results on CIFAR10 with ResNet20/32/44. The accuracy of full-precision (FP) baseline is reported by \citep{NEURIPS2019f8e59f4b}.}
  \begin{center}
    \begin{tabular}{ccccc}
      \toprule[1pt]
      \textbf{Network} & \textbf{Forward} & \textbf{Backward} & \textbf{Test Acc (\%)} & \textbf{FP Acc (\%)}  \\
      \toprule[1pt]
      \multirow{3}{*}{ResNet20} & \multirow{3}{*}{\{$-1$,$+1$\}} & Dorefa & 88.28$\pm$0.81 & \multirow{3}{*}{91.50}  \\
      && MultiFCG & 88.94$\pm$0.46 &  \\
      && RF-DNN & \textbf{89.83$\pm$0.23} &  \\
      \midrule
      \multirow{3}{*}{ResNet32} & \multirow{3}{*}{\{$-1$,$+1$\}} & Dorefa & 90.23$\pm$0.63 & \multirow{3}{*}{92.13}  \\
      && MultiFCG & 89.63$\pm$0.38 &  \\
      && RF-DNN & \textbf{90.75$\pm$0.19} &  \\
      \midrule
      \multirow{3}{*}{ResNet44} & \multirow{3}{*}{\{$-1$,$+1$\}} & Dorefa & 90.71$\pm$0.58 & \multirow{3}{*}{93.56} \\
      && MultiFCG & 90.54$\pm$0.21 & \\
      && RF-DNN & \textbf{91.63$\pm$0.11} & \\
      \bottomrule[1pt]
    \end{tabular}
  \end{center}
  \label{table1}
\end{table}

\begin{table}[thbp]
	\caption{The experimental results on CIFAR100 with ResNet56/110. The accuracy of full-precision (FP) baseline is reported by \citep{NEURIPS2019f8e59f4b}.}
  \begin{center}
    \begin{tabular}{ccccc}
      \toprule[1pt]
      \textbf{Network} & \textbf{Forward} & \textbf{Backward} & \textbf{Test Acc (\%)} & \textbf{FP Acc (\%)}  \\
      \toprule[1pt]
      \multirow{4}{*}{ResNet56} & \multirow{4}{*}{\{$-1$,$+1$\}} & Dorefa & 66.71$\pm$2.32 & \multirow{4}{*}{71.22}  \\
      && MultiFCG & 66.58$\pm$0.37 &  \\
      && FCGrad & 66.56$\pm$0.35 &  \\
      && RF-DNN & \textbf{68.56$\pm$0.32} &  \\
      \midrule
      \multirow{4}{*}{ResNet110} & \multirow{4}{*}{\{$-1$,$+1$\}} & Dorefa & 68.15$\pm$0.50 & \multirow{4}{*}{72.54}  \\
      && MultiFCG & 68.27$\pm$0.14 &  \\
      && FCGrad & 68.74$\pm$0.36 &  \\
      && RF-DNN & \textbf{69.20$\pm$0.28} &  \\
      \bottomrule[1pt]
    \end{tabular}
  \end{center}
  \label{table2}
\end{table}
\begin{table}[t]
	\caption{The experimental results on ImageNet with ResNet18. The accuracy of full-precision (FP) baseline is reported by \citep{NEURIPS2019f8e59f4b}.}
  \begin{center}
    \begin{tabular}{ccccc}
      \toprule[1pt]
      \textbf{Network} & \textbf{Forward} & \textbf{Backward} & \textbf{Test Top1/Top5 (\%)} & \textbf{FP Top1/Top5 (\%)}  \\
      \toprule[1pt]
      \multirow{4}{*}{ResNet18} & \multirow{4}{*}{\{$-1$,$+1$\}} & Dorefa & 58.34$\pm$2.07/81.47$\pm$1.56 & \multirow{4}{*}{69.76/89.08}  \\
      && MultiFCG & 59.47$\pm$0.02/82.41$\pm$0.01 &  \\
      && FCGrad & 59.83$\pm$0.36/82.67$\pm$0.23 &  \\
      && RF-DNN & \textbf{60.83$\pm$0.41/83.54$\pm$0.18} &  \\
      \bottomrule[1pt]
    \end{tabular}
  \end{center}
  \label{table3}
\end{table}

Batch normalization with a batch size of 128 is employed in the learning strategy, and Nesterov momentum of 0.9~\citep{dozat2016incorporating} is used in SGD optimization. For CIFAR, we set the total training epochs to 200 and a weight decay of 0.0005. The learning rate is reduced by a factor of 10 at epoch 80, 150, and 190, starting with an initial value of 0.1. For ImageNet, we set the total training epochs to 100 and use a cosine annealing schedule for the learning rate of each parameter group with a weight decay of 0.0001. All experiments are conducted 5 times, and the statistics of the test accuracies from the last 10/5 epochs are reported for a fair comparison. Hence, we evaluate the accuracy performance in terms of (mean $\pm$ std). Note that we perform standard data augmentation and pre-processing on CIFAR and ImageNet datasets.

In Table~\ref{table1}, Table~\ref{table2}, and Table~\ref{table3}, we use the same the discrete function $Q^1(\cdot)$, parameter settings, and optimizer for fairness in the forward pass. The only difference is the gradient in the backward propagation.
The performance across various models and datasets demonstrates that RF-DNN exhibits significant improvement over other STE methods. The average results of multiple experiments surpass those of other methods, which likely benefit from the alleviation of the gradient mismatch, making the loss function of DNNs more fully descended. Additionally, the minor variances indicate that our training method is relatively stable such that confirming our point of view.

\subsection{Convergence and Stability Analysis}
\label{7.3}

\begin{figure}[htbp]
	\centering
  \subfigure[Comparison of different methods]
	{\includegraphics[width=.49\textwidth]{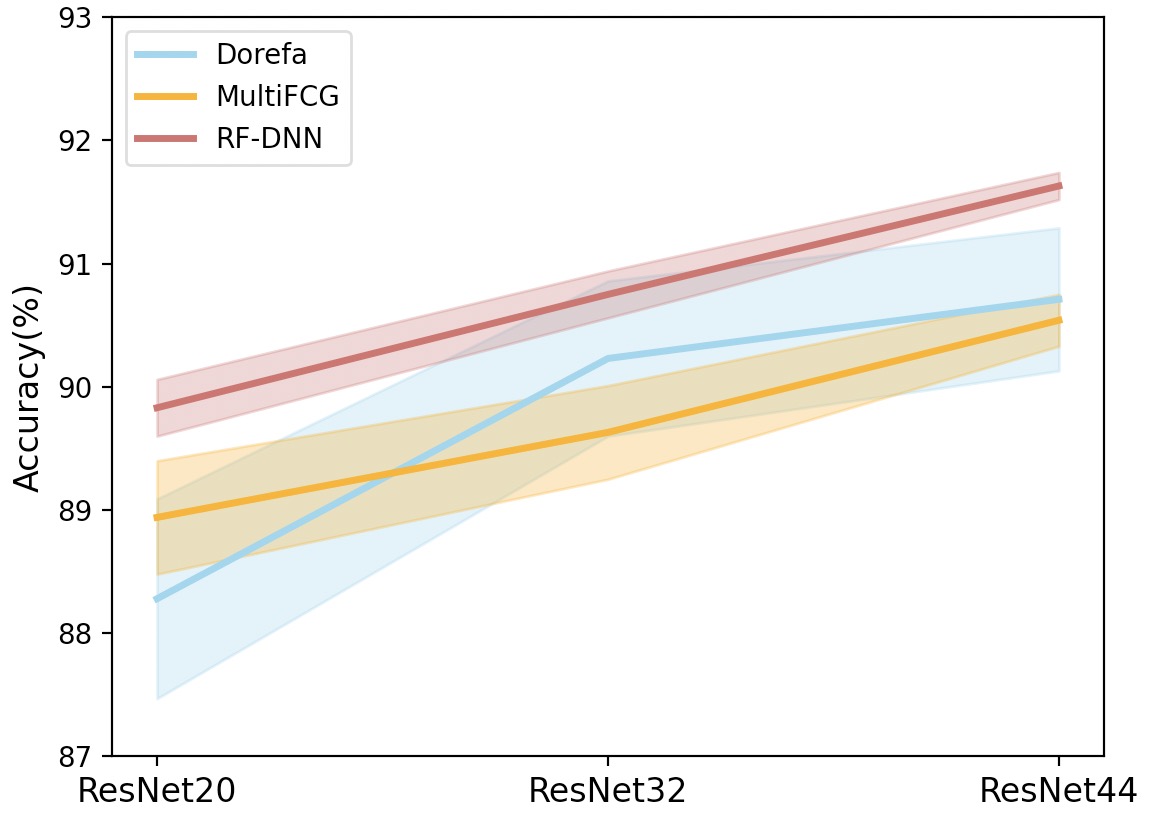}\label{visualizationa}}
  \subfigure[Comparison of different bit widths]
	{\includegraphics[width=.49\textwidth]{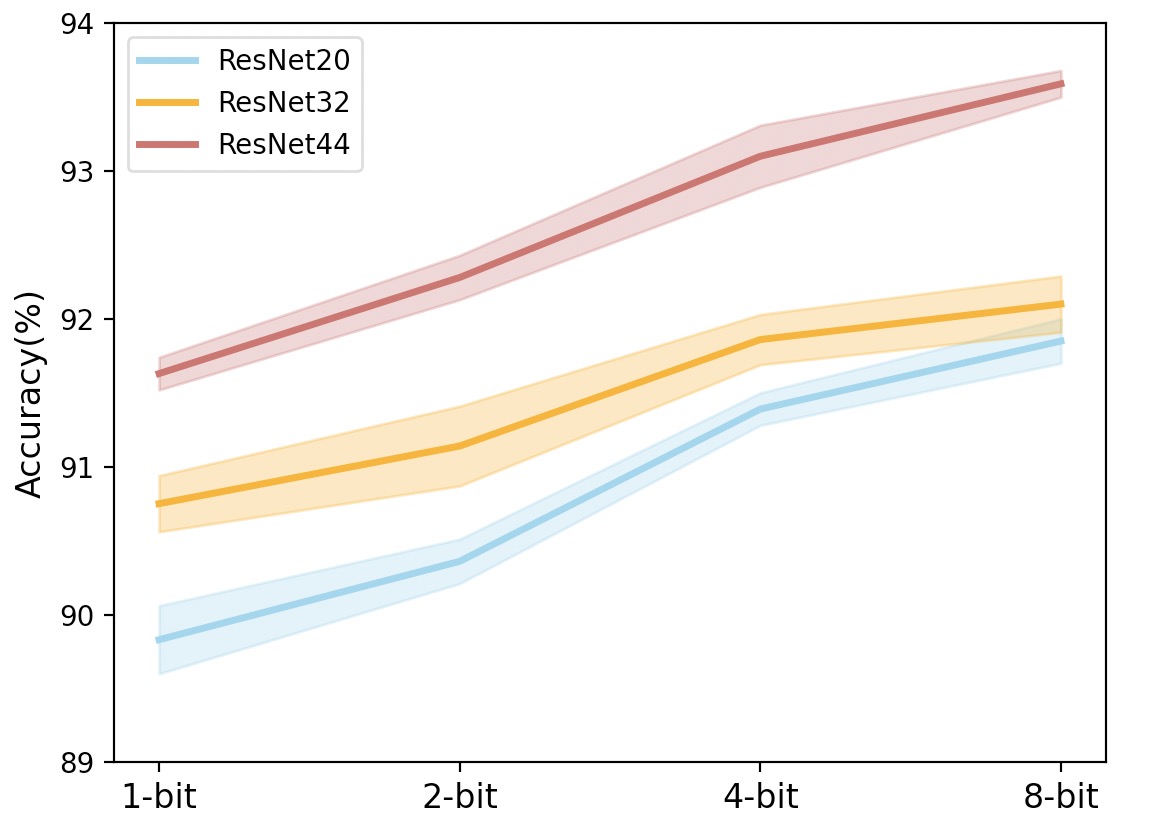}\label{visualizationb}}
	\caption{Accuracy performance (mean $\pm$ std) for ResNet20/32/44 on CIFAR10. The lines and bars represent the mean and standard deviation of the results from different random seeds, respectively. (a) We compare RF-DNN with Dorefa and MultiFCG using 1-bit weight representation, also visualized in Table~\ref{table1}. (b) RF-DNN is presented with different bit-width weight representations. Note that a higher mean and lower deviation typically imply better convergence and stability.}
	\label{visualization}
\end{figure}

\begin{figure}[thbp]
	\centering
  \subfigure[ResNet56]
	{\includegraphics[width=.49\textwidth]{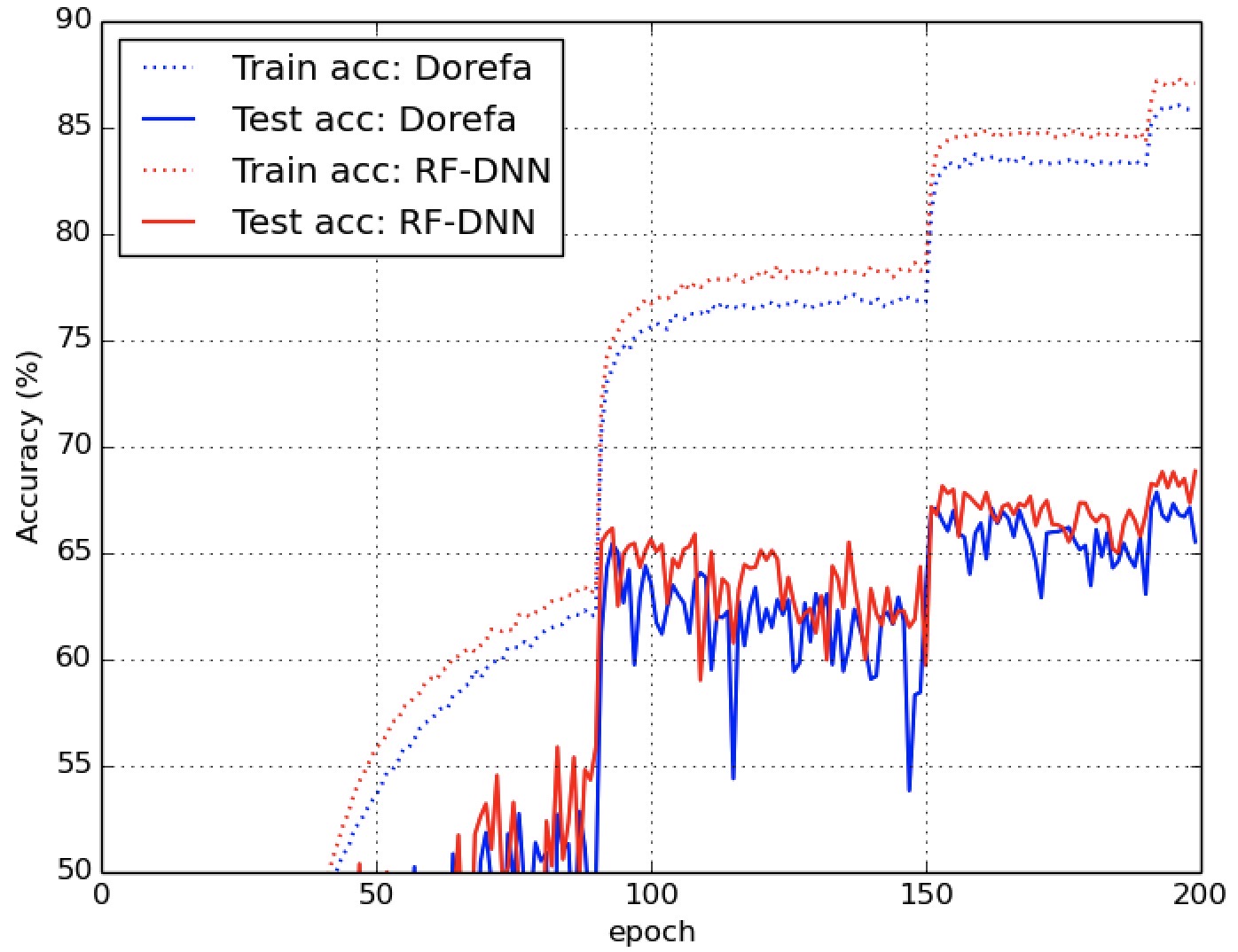}}
  \subfigure[ResNet110]
	{\includegraphics[width=.49\textwidth]{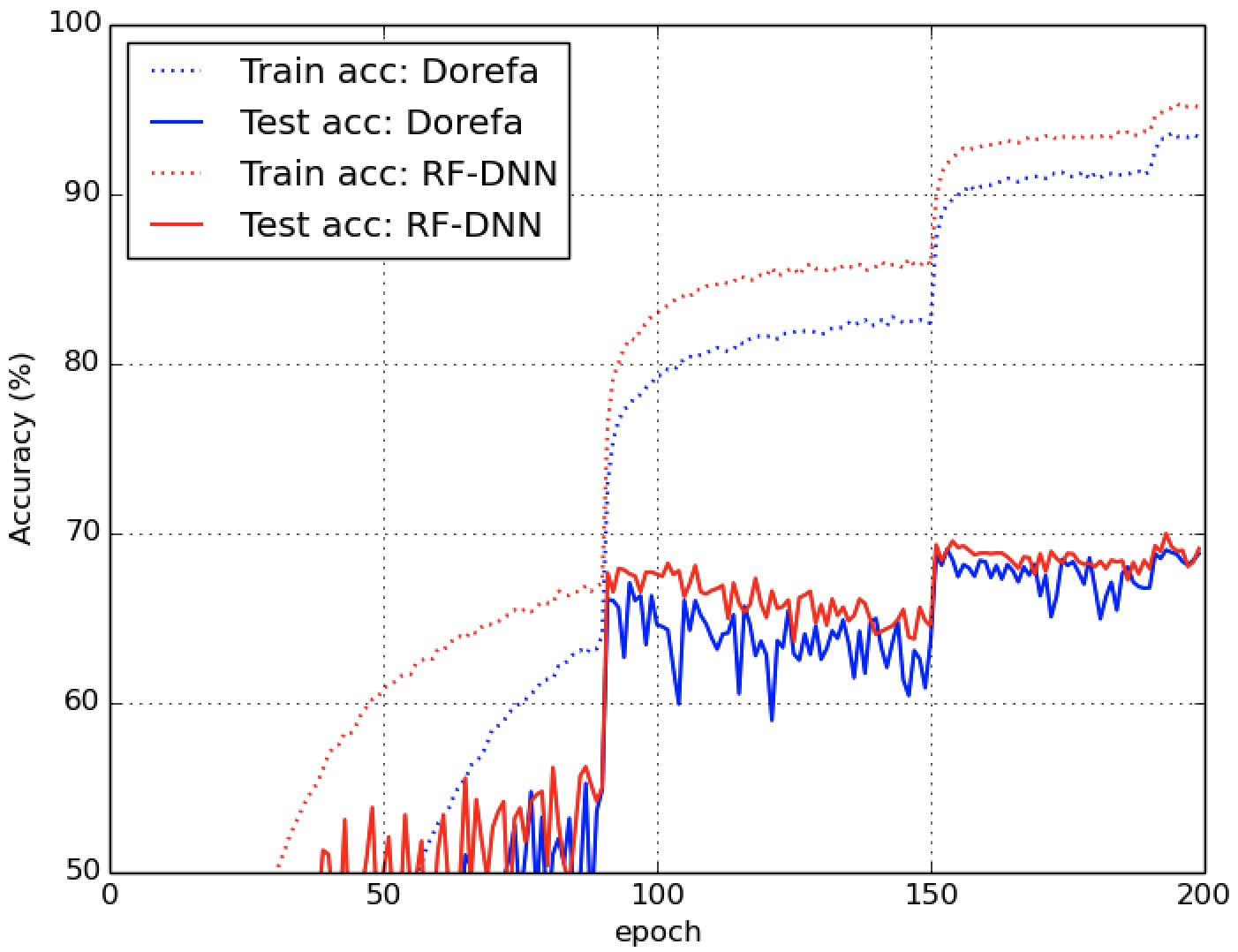}}
	\caption{Training and test curves of ResNet56/110 on CIFAR100 compared between Dorefa and RF-DNN. Intuitively, RF-DNN exhibits more stable training performance than Dorefa.}
	\label{convergence}
\end{figure}

\begin{table}[ht]
	\caption{The classification accuracy results on ImageNet are compared with other training-based methods, including AlexNet~\citep{krizhevsky2012imagenet}, ResNet18, ResNet50 and MobileNet~\citep{howard2017mobilenets}. Note that the accuracy of full-precision baseline is reported by \citet{elhoushi2019deepshift}.}
  \begin{center}
    \begin{tabular}{lcccccc}
      \toprule[1pt]
      \multirow{2}{*}{\textbf{Method}} & \multirow{2}{*}{\textbf{W}} & \multirow{2}{*}{\textbf{A}} & \multicolumn{2}{c}{\textbf{Top-1}} & \multicolumn{2}{c}{\textbf{Top-5}} \\
      \cmidrule{4-7}
      & & & Accuracy & Gap & Accuracy & Gap \\

      \toprule[1pt]
      \textbf{AlexNet} (Original) & 32 & 32 & 56.52\% & - & 79.07\% & - \\
      \midrule
      RF-DNN (ours) & 6 & 32 & \textbf{56.39\%} & \textbf{$-$0.13\%} & \textbf{78.78\%} & \textbf{$-$0.29\%} \\
      DeepShift~\citep{elhoushi2019deepshift} & 6 & 32 & 54.97\% & $-$1.55\% & 78.26\% & $-$0.81\% \\

      \toprule[1pt]
      \textbf{ResNet18} (Original) & 32 & 32 & 69.76\% & - & 89.08\% & - \\
      \midrule
			RF-DNN (ours) & 1 & 32 & \textbf{67.05\%} & \textbf{$-$2.71\%} & \textbf{88.09\%} & \textbf{$-$0.99\%} \\
	  MD~\citep{ajanthan2021mirror} & 1 & 32 & 66.78\% & $-$2.98\% & 87.01\% & $-$2.07\% \\
      ELQ~\citep{zhou2018explicit} & 1 & 32 & 66.21\% & $-$3.55\% & 86.43\% & $-$2.65\% \\
      ADMM~\citep{leng2018extremely} & 1 & 32 & 64.80\% & $-$4.96\% & 86.20\% & $-$2.88\% \\
      QN~\citep{Yang2019CVPR} & 1 & 32 & 66.50\% & $-$3.26\% & 87.30\% & $-$1.78\% \\
      MetaQuant~\citep{NEURIPS2019f8e59f4b} & 1 & 32 & 63.44\% & $-$6.32\% & 84.77\% & $-$4.31\% \\
      RF-DNN (ours) & 4 & 4 & \textbf{66.75\%} & \textbf{$-$3.01\%} & \textbf{87.02\%} & \textbf{$-$2.06\%} \\
      RQ ST~\citep{louizos2018relaxed} & 4 & 4 & 62.46\% & $-$7.30\% & 84.78\% & $-$4.30\% \\

      \toprule[1pt]
      \textbf{ResNet50} (Original) & 32 & 32 & 76.13\% & - & 92.86\% & - \\
      \midrule
			RF-DNN (ours) & 8 & 8 & \textbf{76.07\%} & \textbf{$-$0.06\%} & \textbf{92.87\%} & \textbf{$+$0.01\%} \\
      INT8~\citep{zhu2020towards} & 8 & 8 & 75.87\% & $-$0.26\% & - & - \\

      \toprule[1pt]
      \textbf{MobileNet} (Original) & 32 & 32 & 70.61\% & - & 89.47\% & - \\
      \midrule
      RF-DNN (ours) & 5 & 5 & \textbf{61.32\%} & \textbf{$-$9.29\%} & \textbf{84.08\%} & \textbf{$-$5.39\%} \\
			SR+DR~\citep{gysel2018ristretto} & 5 & 5 & 59.39\% & $-$11.22\% & 82.35\% & $-$7.12\% \\
			RQ ST~\citep{louizos2018relaxed} & 5 & 5 & 56.85\% & $-$13.76\% & 80.35\% & $-$9.12\% \\
      RF-DNN (ours) & 8 & 8 & \textbf{70.76\%} & \textbf{$+$0.15\%} & \textbf{89.54\%} & \textbf{$+$0.07\%} \\
      RQ~\citep{louizos2018relaxed} & 8 & 8 & 70.43\% & $-$0.18\% & 89.42\% & $-$0.05\% \\
      \bottomrule[1pt]
    \end{tabular}
  \end{center}
  \label{table4}
\end{table}

Since standard deviations can reflect the convergence and stability of training to a certain extent, we visualize the data from Table~\ref{table1} in Figure~\ref{visualizationa}. Intuitively, when compared to Dorefa and MultiFCG, our proposed RF-DNN better alleviates perturbations caused by gradient mismatch, leading to more stable performance. Furthermore, we present the accuracy performance of RF-DNN with different bit width weight representations in Figure~\ref{visualizationb}. We observe fairly consistent stability across different bit widths and backbone models.

As depicted in Figure~\ref{convergence}, RF-DNN achieves higher accuracies than Dorefa on CIFAR100 dataset, i.e., 1.25\% higher on the training dataset with ResNet56, 1.85\% higher on the test dataset with ResNet56, 1.97\% higher on the training dataset with ResNet110, and 1.05\% higher on the test dataset with ResNet110. Additionally, the fluctuation of the test curves in Figure~\ref{convergence} indicates that RF-DNN shows tremendous improvement compared to Dorefa in terms of training stability.
From the training curve, our method significantly outperforms Dorefa. However, this superiority needs to be considered in conjunction with the test curve. The accuracy of our method is consistently higher than that of Dorefa in the test curve, thereby indicating an improvement in stability.
The experimental results verify that our theoretical framework is an effective solution against gradient mismatch, further enhancing the training performance of DNNs.

% \subsection{Training Time Analysis}
% \label{7.4}

% In our hardware environment, we record the training time per iteration for our RF-DNN, MultiFCG, and Dorefa with the ResNet20 model in the CIFAR10 dataset. For each training iteration, our RF-DNN costs 35.34 seconds, MultiFCG costs 18.44 seconds, and Dorefa costs 17.81 seconds.
% Compared with MultiFCG and Dorefa, our RF-DNN costs more calculation time, and this part of the extra time mainly comes from solving the constraints of the Ricci flow. Although the training time of our RF-DNN is longer, it is acceptable compared to the improvement brought by our method. And our method also has room for parallel optimization, but it is out of the scope of this paper.

\subsection{Comparisons with Training-based Methods}
\label{7.5}

Here, we compare RF-DNN with several state-of-the-art DNNs, such as DeepShift~\citep{elhoushi2019deepshift}, QN~\citep{Yang2019CVPR}, ADMM~\citep{leng2018extremely}, MetaQuant~\citep{NEURIPS2019f8e59f4b}, INT8~\citep{zhu2020towards}, SR+DR~\citep{gysel2018ristretto}, ELQ~\citep{zhou2018explicit}, MD~\citep{ajanthan2021mirror}, and RQ~\citep{louizos2018relaxed}, all under the same bit width using Equation~(\ref{eq42}) or Equation~(\ref{eq43}).
Note that $\textbf{W}$ and $\textbf{A}$ represent the bit width of weights and activations, respectively, in Table~\ref{table4}.
The experimental results demonstrate that RF-DNN outperforms other recent state-of-the-art training-based methods, which appears to be attributed to our effective solution for addressing gradient mismatch.

\section{Conclusion and Future Work}
\label{8}

Traditional discretized neural networks (DNNs) suggest that both weights and activations can only take low-precision discrete values, reducing the memory footprint compared to full-precision floating-point networks. However, training such networks becomes challenging due to the need to maintain discrete weights. Generally, the gradient w.r.t. discrete weights is approximated using the Straight-Through Estimator (STE), resulting in a \emph{gradient mismatch} compared to the gradient w.r.t. continuous weights.

This paper introduces a novel analysis of the gradient mismatch phenomenon through the lens of duality theory. The mismatch is interpreted as metric perturbations in a Riemannian manifold. Theoretical insights, rooted in information geometry, lead to the construction of the LNE manifold for neural networks. This manifold forms the background to effectively address metric perturbations.
The stability of LNE metrics with the $L^2$-norm perturbation under the Ricci-DeTurck flow is revealed, paving the way for practical introduction of the Ricci flow Discretized Neural Network (RF-DNN). The constraints of the discrete Ricci flow in the LNE manifold are used to alleviate metric perturbations, achieving an exponential convergence rate and providing a compelling solution for DNNs.
Experimental results demonstrate improvements in both the stability and performance of DNNs.

In this paper, information geometry plays a crucial role in combining geometric tool (Ricci flow) with neural networks. For future research, we aim to further explore the connection between neural networks and manifolds, leveraging geometric ideas to address practical challenges in deep learning.

% Acknowledgements should go at the end, before appendices and references

\acks{We thank all reviewers and the editor for excellent contributions.}

% Manual newpage inserted to improve layout of sample file - not
% needed in general before appendices/bibliography.
\newpage
\appendix
\section{Differential Geometry}
\label{app1}
1. Riemann curvature tensor (Rm) is a (1,3)-tensor defined for a 1-form $\omega$:
\[
R^l_{ijk}\omega_l=\nabla_i \nabla_j \omega_k - \nabla_j \nabla_i \omega_k
\]
where the covariant derivative of $F$ satisfies
\[
\nabla_{p} F_{i_{1} \ldots i_{k}}^{j_{1} \ldots j_{l}}=\partial_{p} F_{i_{1} \ldots i_{k}}^{j_{1} \ldots j_{l}}+\sum_{s=1}^{l} F_{i_{1} \ldots i_{k}}^{j_{1} \ldots q \ldots j_{l}} \Gamma_{p q}^{j_{s}}-\sum_{s=1}^{k} F_{i_{1} \ldots q_{\ldots} i_{k}}^{j_{1} \ldots j_{l}} \Gamma_{p i_{s}}^{q}.
\]
In particular, coordinate form of the Riemann curvature tensor is:
\[
R_{i j k}^{l}=\partial_{i} \Gamma_{j k}^{l}-\partial_{j} \Gamma_{i k}^{l}+\Gamma_{j k}^{p} \Gamma_{i p}^{l}-\Gamma_{i k}^{p} \Gamma_{j p}^{l}.
\]
2. Christoffel symbol in terms of an ordinary derivative operator is:
\[
\Gamma^k_{ij}=\frac{1}{2}g^{kl}(\partial_i g_{jl}+\partial_j g_{il}-\partial_l g_{ij}).
\]
3. Ricci curvature tensor (Ric) is a (0,2)-tensor:
\[
R_{ij}=R^p_{pij}.
\]
4. Scalar curvature is the trace of the Ricci curvature tensor:
\[
R=g^{ij}R_{ij}.
\]
5. Lie derivative of $F$ in the direction $\frac{d \varphi(t)}{dt}$:
\[
\mathcal{L}_{\frac{d \varphi(t)}{dt}} F=\left(\frac{d}{dt}\varphi^*(t) F\right)_{t=0}
\]
where $\varphi(t): \mathcal{M} \rightarrow \mathcal{M}$ for $t\in(-\epsilon,\epsilon)$ is a time-dependent diffeomorphism of $\mathcal{M}$ to $\mathcal{M}$.

\section{Notation}
\label{appnotation}

For clarity of definitions in this paper, we list the important notations as shown in Table~\ref{tabnotation}.

\begin{table}[th]
  \begin{center}
    \caption{Definitions of notations}
    \begin{tabular}{|ll|ll|}
    \hline
    $\boldsymbol{W}_i$: & \makecell[l]{weight matrix \\ for the $i$-th layer} & $\boldsymbol{\hat{W}}_i$: & \makecell[l]{discretized weight matrix \\ for the $i$-th layer} \\
    \hline
    $\boldsymbol{w}$: & \makecell[l]{vectorized weights \\ in each layer} & $\boldsymbol{\hat{w}}$: & \makecell[l]{discretized vectorized weights \\ in each layer} \\
    \hline
    $\boldsymbol{a}_i$: & \makecell[l]{activation vector \\ for the $i$-th layer} & $\boldsymbol{\hat{a}}_i$: & \makecell[l]{discretized activation vector \\ for the $i$-th layer} \\
    \hline
    $\boldsymbol{\xi}$: & parameter vector & $\boldsymbol{\hat{\xi}}$: & discretized parameter vector \\
    \hline
    $Q^1$: & 1-bit discrete function & $Q^{k>1}$: & \makecell[l]{k-bit discrete function \\ (over 1-bit)} \\
    \hline
    $\delta$: & \makecell[l]{Euclidean metric \\ (identity matrix)} & $\Phi$: & convex function \\
    \hline
    $g_0$: & LNE metric under Ricci flow & $\bar{g}_0$: & \makecell[l]{LNE metric \\ under Ricci-DeTurck flow} \\
    \hline
    $g$ or $g(t)$: & \makecell[l]{the metrics under Ricci flow} & $\bar{g}$ or $\bar{g}(t)$: & \makecell[l]{the metrics \\ under Ricci-DeTurck flow} \\
    \hline
    $g(0)$: & initial metric under Ricci flow & $\bar{g}(0)$: & \makecell[l]{initial metric \\ under Ricci-DeTurck flow} \\
    \hline
    $d(0)$: & the initial perturbation & $d(t)$: & \makecell[l]{the time-evolving \\ perturbation} \\
    \hline
    $D$: & divergence & $L$ or $\tilde{L}$: & loss function \\
    \hline
    $L_{g_0}$: & Lichnerowicz operator & $L^2$ or $L^\infty$: & norm \\
    \hline
    $\partial$: & partial derivative & $\nabla$: & covariant derivative \\
    \hline
    $\mathcal{L}$: & Lie derivative & $\Delta_{g_0}$: & the Laplacian \\
    \hline
    $\operatorname{Rm}$: & Riemann curvature tensor & $f$: & nonlinear function \\
    \hline
    $\operatorname{Ric}$: & Ricci curvature tensor & $\mathcal{D}[\operatorname{Ric}]$: & \makecell[l]{the linearization of \\ the Ricci curvature tensor} \\
    \hline
    $\varphi^*$: & pullback & $\phi_*$: & pushforward \\
    \hline
     $B(x,r)$: & \makecell[l]{the ball with a radius $r$ \\ and a point $x \in \mathcal{M}$} & $\mathcal{B}_{L^2}(\bar{g}_0, \epsilon)$: & \makecell[l]{the $\epsilon$-ball with respect to \\ the $L^2$-norm induced by $\bar{g}_0$ \\ and centred at $\bar{g}_0$}\\
    \hline
    \end{tabular}
    \label{tabnotation}
  \end{center}
\end{table}

\section{Proof of the Ricci Flow}
\label{app2}

\subsection{Proof of Lemma~\ref{le1}}
\label{app21}

\begin{lemma}
	\label{linear}
	The linearization of the Ricci curvature tensor is given by
	\[
	\mathcal{D}[\operatorname{Ric}](h)_{i j}=-\frac{1}{2} g^{p q}(\nabla_{p} \nabla_{q} h_{i j}+\nabla_{i} \nabla_{j} h_{p q}-\nabla_{q} \nabla_{i} h_{jp}-\nabla_{q} \nabla_{j} h_{i p}).
	\]
\end{lemma}
\begin{proof}
	Based on Appendix~\ref{app1}, we have
	\[
	\nabla_{q} \nabla_{i} h_{j p} =\nabla_{i} \nabla_{q} h_{j p}-R_{q i j}^{r} h_{r p}-R_{q i p}^{r} h_{j m}.
	\]
	Combining with Lemma~\ref{linear}, we can obtain the deformation equation because of $\nabla g=0$,
	\[
	\begin{aligned}
	\mathcal{D}[-2 \mathrm{Ric}](h)_{i j}=& g^{p q} \nabla_{p} \nabla_{q} h_{i j}+\nabla_{i}\left(\frac{1}{2} \nabla_{j} h_{p q}-\nabla_{q} h_{j p}\right)+\nabla_{j}\left(\frac{1}{2} \nabla_{i} h_{p q}-\nabla_{q} h_{i p}\right)+O(h_{ij}) \\
	=& g^{p q} \nabla_{p} \nabla_{q} h_{i j}+\nabla_{i} V_{j}+\nabla_{j} V_{i}+O(h_{ij}).
	\end{aligned}
	\]
 The proof is completed.
\end{proof}

\subsection{Description of the DeTurck Trick}
\label{app22}

Based on the chain rule for the Lie derivative in Appendix~\ref{app1}, we can calculate
\[
\begin{aligned}
\frac{\partial}{\partial t} g(t) &=\frac{\partial\left(\varphi^{*}(t) \bar{g}(t)\right)}{\partial t} \\
&=\left(\frac{\partial\left(\varphi^{*}(t+\tau) \bar{g}(t+\tau)\right)}{\partial \tau}\right)_{\tau=0} \\
&=\left(\varphi^{*}(t) \frac{\partial \bar{g}(t+\tau)}{\partial \tau}\right)_{\tau=0}+\left(\frac{\partial\left(\varphi^{*}(t+\tau) \bar{g}(t)\right)}{\partial \tau}\right)_{\tau=0} \\
&=\varphi^{*}(t) \frac{\partial}{\partial t}\bar{g}(t)+\varphi^{*}(t) \mathcal{L}_{\frac{\partial \varphi(t)}{\partial t}} \bar{g}(t)
\end{aligned}
\]
where $\frac{\partial \varphi(t)}{\partial t}$ is equal to $V(t)$~\citep{sheridan2006hamilton}. With the help of Equation~(\ref{ricci}), we have the following expression for the pullback metric $g(t)$
\begin{equation}
\frac{\partial}{\partial t} g(t)=\varphi^{*}(t) \frac{\partial}{\partial t}\bar{g}(t)+\varphi^{*}(t) \mathcal{L}_{\frac{\partial \varphi(t)}{\partial t}} \bar{g}(t)=-2 \operatorname{Ric}(\varphi^*(t) \bar{g}(t))=-2 \varphi^*(t) \operatorname{Ric}(\bar{g}(t)).
\end{equation}
The diffeomorphism invariance of the Ricci curvature tensor is used in the last step. The above equation is equivalent to
\[
\frac{\partial}{\partial t}\bar{g}(t)=-2 \operatorname{Ric}(\bar{g}(t))- \mathcal{L}_{\frac{\partial \varphi(t)}{\partial t}} \bar{g}(t).
\]
Based on Definition~\ref{nabla}, we further yield
\[
\frac{\partial}{\partial t}\bar{g}(t)=-2 \operatorname{Ric}(\bar{g}(t))- \nabla_i V_j - \nabla_j V_i.
\]
\begin{definition}
	\label{nabla}
	\citep{sheridan2006hamilton} On a Riemannian manifold $(\mathcal{M}, g)$, we have
	\[
	(\mathcal{L}_X g)_{ij} = \nabla_i X_j + \nabla_j X_i,
	\]
	where $\nabla$ denotes the Levi-Civita connection of the metric $g$, for any vector field $X$.
\end{definition}

\subsection{Curvature Explosion at Singularity}
\label{app24}

In general, we present the behavior of Ricci flow in finite time and show that the evolution of the curvature is close to divergence. The core demonstration is followed with Theorem~\ref{thm4}.

\begin{theorem}
	\label{thm2}
	\citep{sheridan2006hamilton} Given a smooth Riemannian metric $g_0$ on a closed manifold $\mathcal{M}$, there exists a maximal time interval $[0, T)$ such that a solution $g(t)$ of the Ricci flow, with $g(0) = g_0$, exists and is smooth on $[0, T)$, and this solution is unique.
\end{theorem}

\begin{theorem}
	\label{thm3}
	Let $\mathcal{M}$ be a closed manifold and $g(t)$ a smooth time-dependent metric on $\mathcal{M}$, defined for $t \in [0, T)$. If there exists a constant $C < \infty$ for all $x \in \mathcal{M}$ such that
	\begin{equation}
	\int_{0}^{T}\left|\frac{\partial}{\partial t} g_x(t)\right|_{g(t)} d t  \leq C,
	\end{equation}
	then the metrics $g(t)$ converge uniformly as $t$ approaches $T$ to a continuous metric $g(T)$ that is uniformly equivalent to $g(0)$ and satisfies
	\begin{equation}
  \label{close}
	e^{-C} g_x(0) \leq g_x(T) \leq e^C g_x(0).
  \end{equation}
\end{theorem}

\begin{proof}
  Considering any $x \in \mathcal{M}$, $t_0 \in [0, T)$, $V \in T_x \mathcal{M}$, we have
  \[
  \begin{aligned}
  \left|\log \left(\frac{g_x (t_0)(V, V)}{g_x (0)(V, V)}\right)\right| &=\left|\int_{0}^{t_{0}} \frac{\partial}{\partial t}\left[\log g_x (t)(V, V)\right] d t\right| \\
  &=\left|\int_{0}^{t_{0}} \frac{\frac{\partial}{\partial t} g_x (t)(V, V)}{g_x (t)(V, V)} d t\right| \\
  & \leq \int_{0}^{t_{0}}\left|\frac{\partial}{\partial t} g_x (t) \left(\frac{V}{|V|_{g(t)}}, \frac{V}{|V|_{g(t)}}\right)\right| d t \\
  & \leq \int_{0}^{t_{0}}\left|\frac{\partial}{\partial t} g_x (t)\right|_{g(t)} d t \\
  & \leq C.
  \end{aligned}
  \]
  By exponentiating both sides of the above inequality, we have
  \[
  e^{-C} g_x(0)(V, V) \leq g_x(t_0)(V, V) \leq e^C g_x(0)(V, V).
  \]
  This inequality can be rewritten as
  \[
  e^{-C} g_x(0) \leq g_x(t_0)(V, V) \leq e^C g_x(0)(V, V)
  \]
  because it holds for any $V$. Thus, the metrics $g(t)$ are uniformly equivalent to $g(0)$.

  Consequently, we have the well-defined integral:
  \[
  g_x(T) - g_x(0) = \int_{0}^{T}\frac{\partial}{\partial t} g_x (t) d t.
  \]
  We can show that this integral is well-defined from two perspectives. Firstly, as long as the metrics are smooth, the integral exists. Secondly, the integral is absolutely integrable. Based on the norm inequality induced by $g(0)$, we can obtain
  \[
  |g_x(T) - g_x(t)|_{g(0)} \leq \int_{t}^{T}\left|\frac{\partial}{\partial t} g_x (t)\right|_{g(0)} d t.
  \]
  For each $x \in \mathcal{M}$, the above integral will approach zero as $t$ approaches $T$. Since $\mathcal{M}$ is compact, the metrics $g(t)$ converge uniformly to a continuous metric $g(T)$ which is uniformly equivalent to $g(0)$ on $\mathcal{M}$. Moreover, we can show that
  \[
  e^{-C} g_x(0) \leq g_x(T) \leq e^C g_x(0).
  \]
  The proof is completed.
\end{proof}

\begin{corollary}
	\label{cor1}
	Let $(\mathcal{M}, g(t))$ be a solution of the Ricci flow on a closed manifold. If $|\operatorname{Rm}|_{g(t)}$ is bounded on a finite time $[0, T)$, then $g(t)$ converges uniformly as $t$ approaches $T$ to a continuous metric
	$g(T)$ which is uniformly equivalent to $g(0)$.
\end{corollary}

\begin{proof}
	The bound on $|\operatorname{Rm}|_{g(t)}$ implies one on $|\operatorname{Ric}|_{g(t)}$. Based on Equation~(\ref{ricci}), we can extend the bound on $|\frac{\partial}{\partial t}g(t)|_{g(t)}$. Therefore, we obtain an integral of a bounded quantity over a finite interval is also bounded, by Theorem~\ref{thm3}. The proof is completed.
\end{proof}

\begin{theorem}
	\label{thm4}
	If $g_0$ is a smooth metric on a compact manifold $\mathcal{M}$, the Ricci flow with $g(0) = g_0$
	has a unique solution $g(t)$ on a maximal time interval $t\in [0, T)$. If $T < \infty$, then
	\begin{equation}
	\lim _{t \rightarrow T}\left(\sup _{x \in \mathcal{M}}|\operatorname{Rm}_x(t)|\right)=\infty.
	\end{equation}
\end{theorem}

\begin{proof}
	For a contradiction, we assume that $|\operatorname{Rm}_x(t)|$ is bounded by a constant. It follows from Corollary~\ref{cor1} that the metrics $g(t)$ converges smoothly to a smooth metric $g(T)$. Based on Theorem~\ref{thm2}, it is possible to find a solution to the Ricci flow on $t \in [0, T)$, as the smooth metric $g(T)$ is uniformly equivalent to the initial metric $g(0)$.

	Hence, we can extend the solution of the Ricci flow after the time point $t=T$, which contradicts the choice of $T$ as the maximal time for the existence of the Ricci flow on $[0,T)$. In other words, $|\operatorname{Rm}_x(t)|$ is unbounded. The proof is completed.
\end{proof}

As approaching the singular time $T$, the Riemann curvature $|\operatorname{Rm}|_{g(t)}$ becomes no longer convergent and tends to explode.

\section{Proof of All Time Convergence in LNE Manifolds}
\label{applong}

\subsection{Finite Time Stability}
\label{fts}
We first prove the finite-time stability of LNE manifolds.

\begin{lemma}
	\label{lem2}
	\citep{bamler2010stability,bamler2011stability} Let $(\mathcal{M}^n, \bar{g}_0)$ be a complete Ricci-flat $n$-manifold. If $\bar{g}(0)$ is a metric satisfying $\|\bar{g}(0)-\bar{g}_0 \|_{L^{\infty}} < \epsilon$ where $\epsilon > 0$, then there exists a constant $C < \infty$ and a unique Ricci–DeTurck flow $\bar{g}(t)$ that satisfies
	\begin{equation}
	\|\bar{g}(t)-\bar{g}_0 \|_{L^{\infty}} < C\|\bar{g}(0)-\bar{g}_0 \|_{L^{\infty}} < C\cdot\epsilon.
  \end{equation}
\end{lemma}

\begin{corollary}
	\label{cor2}
  Let $(\mathcal{M}^n, \bar{g}_0)$ be the LNE $n$-manifold. For a Ricci–DeTurck flow $\bar{g}(t)$ on a maximal time interval $t \in [0, T)$ and $k \in \mathbb{N}$, there exists constants $C_k=C_k(\bar{g}_0, T)$ such that
	\begin{equation}
	\|\nabla^k d(t)\|_{L^2} \leq C_k \cdot t^{-k/2}
	\end{equation}
 where $d(t)=\bar{g}(t)-\bar{g}_0$ is the time-evolving perturbation.
\end{corollary}
\begin{proof}
When Lemma~\ref{lem2} is satisfied in a finite time, based on~\citep{deruelle2021stability}, the Ricci-DeTurck flow with the LNE metric w.r.t. the $L^2$-norm perturbation exists. The proof is completed.
\end{proof}

Corollary~\ref{cor2} guarantees the finite time existence of the Ricci-DeTurck flow w.r.t. $L^2$-norm perturbations and provides the necessary premise for proving its all time convergence.

\subsection{All time Stability}
\label{ats}
Then, we prove the all-time stability of LNE manifolds. By rewriting the Ricci-DeTurck flow~(\ref{newdeturck}) as an evolution of the difference $d(t):=\bar{g}(t)-\bar{g}_0$, we have
\begin{equation}
\begin{aligned}
&\frac{\partial}{\partial t} d(t)=\frac{\partial}{\partial t} \bar{g}(t)=-2 \operatorname{Ric}(\bar{g}(t))+2 \operatorname{Ric}(\bar{g}_0)+\mathcal{L}_{\frac{\partial \varphi'(t)}{\partial t}} \bar{g}_0-\mathcal{L}_{\frac{\partial \varphi(t)}{\partial t}} \bar{g}(t) \\
&=\Delta d(t)+\operatorname{Rm}*d(t)+F_{\bar{g}^{-1}} * \nabla^{\bar{g}_0} d(t) * \nabla^{\bar{g}_0} d(t)+\nabla^{\bar{g}_0}\left(G_{\Gamma(\bar{g}_0)} * d(t) * \nabla^{\bar{g}_0} d(t)\right),
\label{deturck2}
\end{aligned}
\end{equation}
where the tensors $F$ and $G$ depend on $\bar{g}^{-1}$ and $\Gamma(\bar{g}_0)$. Note that $\bar{g}_0$ is the LNE metric which satisfies the above formula.

In the follwing, we denote $\|\cdot\|_{L^2}$ or $\|\cdot\|_{L^{\infty}}$ as the $L^2$-norm or $L^{\infty}$-norm w.r.t. the LNE metric $\bar{g}_0$, and mark generic constants as $C$ or $C_1$.

\begin{lemma}
	\label{lem4}
	Let $\bar{g}(t)$ be a Ricci–DeTurck flow on a maximal time interval $t \in (0,T)$ in an $L^2$-neighbourhood of $\bar{g}_0$.
	We have the following estimate:
	\begin{equation}
	\left\|\frac{\partial}{\partial t} d_{0}(t)\right\|_{L^2} \leq C\left\|\nabla^{\bar{g}_{0}(t)}\left(d(t)-d_{0}(t)\right)\right\|_{L^{2}}^{2}.
  \end{equation}
\end{lemma}

\begin{proof}
	According to the Hardy inequality~\citep{minerbe2009weighted}, we have the same proofs by referring the details~\citep{deruelle2021stability}.
\end{proof}

To establish the all time stability of LNE metrics under Ricci–DeTurck flow, we need to construct $\bar{g}_0(t)$ as a family of Ricci-flat reference metrics with $\frac{\partial}{\partial t} \bar{g}_0(t)=O((\bar{g}(t)-\bar{g}_0(t))^2)$. Let
\[
\mathcal{F}=\left\{\bar{g}(t) \in \mathcal{M}^n\;\big|\;2 \operatorname{Ric}(\bar{g}(t))+ \mathcal{L}_{\frac{\partial \varphi(t)}{\partial t}} \bar{g}(t)=0\right\}
\]
be the set of stationary points under the Ricci-DeTurck flow. Then, we establish a manifold via an $L^2$-neighbourhood $\mathcal{U}$ of integral $\bar{g}_0$ in the space of metrics:
\begin{equation}
\tilde{\mathcal{F}} = \mathcal{F} \cap \mathcal{U}.
\label{f}
\end{equation}
For all $\bar{g} \in \tilde{\mathcal{F}}$, the terms $\operatorname{Ric}(\bar{g}(t))=0$ and $\mathcal{L}_{\frac{\partial \varphi(t)}{\partial t}} \bar{g}(t)=0$ hold individually, as established in the previous work~\citep{deruelle2021stability}.

\begin{theorem}
	\label{thm5}
	Let $(\mathcal{M}^n, \bar{g}_0)$ be the LNE $n$-manifold which is linearly stable and integrable. Then, there exists a constant $\alpha_{\bar{g}_0}$ satisfying
	\begin{equation}
	\left(\Delta d(t)+\operatorname{Rm}(\bar{g}_0)*d(t), d(t)\right)_{L^{2}} \leq -\alpha_{\bar{g}_0}\left\|\nabla^{\bar{g}_0} d(t)\right\|_{L^{2}}^{2}
  \end{equation}
	for all $\bar{g}(t) \in \tilde{\mathcal{F}}$ whose definition is given in Equation~(\ref{f}).
\end{theorem}

\begin{proof}
	The similar proofs can be found in \citep{devyver2014gaussian} with some minor modifications. Due to the linear stability requirement of LNE manifolds in Definition~\ref{def3} and Definition~\ref{def4}, $-L_{\bar{g}_0}$ is non-negative. Then there exists a positive constant $\alpha_{\bar{g}_0}$ satisfying
	\[
	\alpha_{\bar{g}_0}\left(-\Delta d(t), d(t)\right)_{L^{2}} \leq \left(-\Delta d(t)-\operatorname{Rm}(\bar{g}_0)*d(t), d(t)\right)_{L^{2}}.
	\]
	By Taylor expansion, we repeatedly use elliptic regularity and Sobolev embedding~\citep{pacini2010desingularizing} to obtain the estimate. The proof is completed.
\end{proof}

\begin{corollary}
	\label{cor3}
	Let $(\mathcal{M}^n, \bar{g}_0)$ be the LNE $n$-manifold which is integrable. For a Ricci–DeTurck flow $\bar{g}(t)$ on a maximal time interval $t \in [0, T]$, if it satisfies $\|\bar{g}(t)-\bar{g}_0 \|_{L^{\infty}} < \epsilon$ where $\epsilon > 0$, then there exists a constant $C < \infty$ for $t \in [0, T]$ such that the evolution inequality satisfies
	\begin{equation}
	\|d(t) - d_{0}(t)\|^2_{L^2} \geq C \int_{0}^{T}\left\|\nabla^{\bar{g}_{0}(t)}\left(d(t)-d_{0}(t)\right)\right\|_{L^{2}}^{2} \mathrm{d} t.
  \end{equation}
\end{corollary}

\begin{proof}
Based on Equation~(\ref{deturck2}), we know
\[
\begin{aligned}
\frac{\partial}{\partial t} (d(t)-d_0)=&\Delta (d(t)-d_0)+\operatorname{Rm}*(d(t)-d_0) \\
&+F_{\bar{g}^{-1}} * \nabla^{\bar{g}_0} (d(t)-d_0) * \nabla^{\bar{g}_0} (d(t)-d_0) \\ &+\nabla^{\bar{g}_0}\left(G_{\Gamma(\bar{g}_0)} * (d(t)-d_0) * \nabla^{\bar{g}_0} (d(t)-d_0)\right).
\end{aligned}
\]
Followed by Lemma~\ref{lem4} and Theorem~\ref{thm5}, we further obtain
\[
\begin{aligned}
\frac{\partial}{\partial t} \| d(t) -d_{0} \|_{L^{2}}^{2}=& 2\left(\Delta (d(t)-d_0)+\operatorname{Rm}*(d(t)-d_0), d(t)-d_{0}\right)_{L^{2}} \\
&+\left(F_{\bar{g}^{-1}} * \nabla^{\bar{g}_0} (d(t)-d_0) * \nabla^{\bar{g}_0} (d(t)-d_0), d(t)-d_{0}\right)_{L^{2}} \\
&+\left(\nabla^{\bar{g}_0}\left(G_{\Gamma(\bar{g}_0)} * (d(t)-d_0) * \nabla^{\bar{g}_0} (d(t)-d_0)\right), d(t)-d_{0}\right)_{L^{2}} \\
&+\left(d(t)-d_{0}, \frac{\partial}{\partial t} d_{0}(t)\right)_{L^{2}}+\int_{\mathcal{M}}\left(d(t)-d_{0}\right) *\left(d(t)-d_{0}\right) * \frac{\partial}{\partial t} d_{0}(t) \mathrm{d} \mu \\
\leq &-2 \alpha_{\bar{g}_0}\left\|\nabla^{\bar{g}_0}\left(d(t)-d_{0}\right)\right\|_{L^{2}}^{2} \\
&+C\left\|\left(d(t)-d_{0}\right)\right\|_{L^{\infty}}\left\|\nabla^{\bar{g}_{0}}\left(d(t)-d_{0}\right)\right\|_{L^{2}}^{2} \\
&+\left\|\frac{\partial}{\partial t} d_{0}(t)\right\|_{L^{2}}\left\|d(t)-d_{0}\right\|_{L^{2}} \\
\leq &\left(-2 \alpha_{\bar{g}_{0}}+C \cdot \epsilon\right)\left\|\nabla^{\bar{g}_{0}}\left(d(t)-d_{0}\right)\right\|_{L^{2}}^{2}.
\end{aligned}
\]
Let $\epsilon$ be a small enough constant that $-2 \alpha_{\bar{g}_{0}}+C \cdot \epsilon < 0$ holds, we can find
\[
\frac{\partial}{\partial t} \| d(t) -d_{0} \|_{L^{2}}^{2} \leq -C \left\|\nabla^{\bar{g}_{0}}\left(d(t)-d_{0}\right)\right\|_{L^{2}}^{2}
\]
holds. The proof is completed.
\end{proof}

\subsection{Proof of Theorem~\ref{thm6}}
\label{key}

By Lemma~\ref{lem2}, we have a constant $\epsilon_2>0$ such that $d(t) \in \mathcal{B}_{L^2}(0, \epsilon_2)$ holds. By Lemma~\ref{lem4} (in the second step) and Corollary~\ref{cor3} (in the third step), we can obtain
\[
\begin{aligned}
	&\left\|d_{0}(T)\right\|_{L^{2}} \leq C \int_{1}^{T}\left\|\frac{\partial}{\partial t} d_{0}(t)\right\|_{L^{2}} \mathrm{d} t \\
	&\quad \leq C \int_{1}^{T}\left\|\nabla^{\bar{g}_{0}}\left(d(t)-d_{0}(t)\right)\right\|_{L^{2}}^{2} \mathrm{d} t \\
	&\quad \leq C\left\|d(1)-d_{0}(1)\right\|_{L^{2}}^{2} \leq C\|d(1)\|_{L^{2}}^{2} \leq C \cdot\left(\epsilon_2\right)^{2}.
\end{aligned}
\]
Furthermore, we can obtain from the above formulas
\[
\left\|d(T)-d_0(T)\right\|_{L^{2}} \leq \|d(1)-d_0(1)\|_{L^{2}} \leq C \cdot \epsilon_2.
\]
By the triangle inequality, we get
\[
\left\|d(T)\right\|_{L^{2}} \leq C \cdot\left(\epsilon_2\right)^{2} + C \cdot \epsilon_2.
\]
Followed by Corollary~\ref{cor2} and Lemma~\ref{lem4}, $T$ should be pushed further outward, i.e.,
\[
\lim_{t \rightarrow +\infty}\sup\left\|\frac{\partial}{\partial t} d_{0}(t)\right\|_{L^{2}} \leq \lim_{t \rightarrow +\infty}\sup\left\|\nabla^{\bar{g}_{0}}\left(d(t)-d_{0}(t)\right)\right\|_{L^{2}}^{2}=0.
\]
Thus, as $t$ approaches $+\infty$ based on the elliptic regularity, $\bar{g}(t)$ will converge to $\bar{g}(\infty)=\bar{g}_0+d_0(\infty)$. In other words, $d(t)-d_0(t)$ will converge to $0$ as $t$ approaches $+\infty$ w.r.t. all Sobolev norms~\citep{minerbe2009weighted},
\[
\lim_{t \rightarrow +\infty}\left\|d(t)-d_{0}(t)\right\|_{L^{2}} \leq \lim_{t \rightarrow +\infty}C\left\|\nabla^{\bar{g}_{0}}\left(d(t)-d_{0}(t)\right)\right\|_{L^{2}}=0.
\]

Any Ricci-DeTurck flow that starts close to the LNE metric exists for all time, and it will converge to the LNE metric, as discussed in~\citep{deruelle2021stability}.

\section{Proof of the Information Geometry}
\label{app3}

\subsection{Proof of Theorem~\ref{thm7}}
\label{app31}

The LNE divergence can be defined between two nearby points $\boldsymbol{\xi}$ and $\boldsymbol{\xi}'$, where the first derivative of the LNE divergence w.r.t. $\boldsymbol{\xi}'$ is:
\[
\begin{aligned}
&\partial_{\boldsymbol{\xi}'} D_{LNE}[\boldsymbol{\xi}':\boldsymbol{\xi}] \\
&= \sum_i \left[\partial_{\boldsymbol{\xi}'} \frac{1}{\tau^2}\log\cosh(\tau\xi'_i)- \partial_{\boldsymbol{\xi}'} \frac{1}{\tau^2}\log\cosh(\tau\xi_i)- \frac{1}{\tau}\partial_{\boldsymbol{\xi}'}(\xi'_i-\xi_i)\tanh(\tau\xi_i)\right] \\
& = \sum_i \partial_{\boldsymbol{\xi}'} \frac{1}{\tau^2}\log\cosh(\tau\xi'_i) - \frac{1}{\tau}\tanh(\tau\boldsymbol{\xi}).
\end{aligned}
\]
The second derivative of the LNE divergence w.r.t. $\boldsymbol{\xi}'$ is:
\[
\partial^2_{\boldsymbol{\xi}'} D_{LNE}[\boldsymbol{\xi}':\boldsymbol{\xi}] =\sum_i \partial^2_{\boldsymbol{\xi}'} \frac{1}{\tau^2}\log\cosh(\tau\xi'_i).
\]
We deduce the Taylor expansion of the LNE divergence at $\boldsymbol{\xi}'=\boldsymbol{\xi}$:
\[
\begin{aligned}
D_{LNE}[\boldsymbol{\xi}':\boldsymbol{\xi}] &\approx D_{LNE}[\boldsymbol{\xi}:\boldsymbol{\xi}]+\left(\sum_i \partial_{\boldsymbol{\xi}'} \frac{1}{\tau^2}\log\cosh(\tau\xi'_i) - \frac{1}{\tau}\tanh(\tau\boldsymbol{\xi})\right)^\top \bigg|_{\boldsymbol{\xi}'=\boldsymbol{\xi}} d\boldsymbol{\xi} \\
&+\frac{1}{2}d\boldsymbol{\xi}^\top \left(\sum_i \partial^2_{\boldsymbol{\xi}'} \frac{1}{\tau^2}\log\cosh(\tau\xi'_i)\right) \bigg|_{\boldsymbol{\xi}'=\boldsymbol{\xi}} d\boldsymbol{\xi}\\
& =0 + 0 + \frac{1}{2\tau^2}d\boldsymbol{\xi}^\top \partial \left[ \frac{\partial\cosh(\tau\boldsymbol{\xi})}{\cosh(\tau\boldsymbol{\xi})}\right] d\boldsymbol{\xi} \\
&= \frac{1}{2\tau^2}d\boldsymbol{\xi}^\top  \frac{\partial^2\cosh(\tau\boldsymbol{\xi})\cosh(\tau\boldsymbol{\xi})-\partial\cosh(\tau\boldsymbol{\xi})\partial\cosh(\tau\boldsymbol{\xi})^\top}{\cosh^2(\tau\boldsymbol{\xi})} d\boldsymbol{\xi} \\
&= \frac{1}{2\tau^2}d\boldsymbol{\xi}^\top  \left(\frac{\partial^2\cosh(\tau\boldsymbol{\xi})}{\cosh(\tau\boldsymbol{\xi})} - \tau^2 \left[\frac{\sinh(\tau\boldsymbol{\xi})}{\cosh(\tau\boldsymbol{\xi})}\right]\left[\frac{\sinh(\tau\boldsymbol{\xi})}{\cosh(\tau\boldsymbol{\xi})}\right]^\top\right) d\boldsymbol{\xi} \\
&=\frac{1}{2}\sum_{i,j} \left[\delta_{ij} -\left(\tanh(\tau\boldsymbol{\xi})\tanh(\tau\boldsymbol{\xi})^\top\right)_{ij} d\xi_i d\xi_j\right] .
\end{aligned}
\]

\subsection{Proof of Lemma~\ref{lem5}}
\label{app32}

We would like to know in which direction minimizes the loss function with the constraints of the LNE divergence, so that we do the minimization:
\[
d\boldsymbol{\xi}^{*}=\underset{d\boldsymbol{\xi} \text { s.t. } D_{LNE}[\boldsymbol{\xi}:\boldsymbol{\xi}+d\boldsymbol{\xi}]=c}{\arg \min } L(\boldsymbol{\xi}+d\boldsymbol{\xi})
\]
where $c$ is the constant. The loss function descends along the manifold with constant speed, regardless the curvature.

Furthermore, we can write the minimization in Lagrangian form. Combined with Theorem~\ref{thm7}, the LNE divergence can be approximated by its second order Taylor expansion. Approximating $L(\boldsymbol{\xi}+d\boldsymbol{\xi})$ with it first order Taylor expansion, we get:
\[
\begin{aligned}
d\boldsymbol{\xi}^* &=\underset{d\boldsymbol{\xi}}{\arg \min}\ L(\boldsymbol{\xi}+d\boldsymbol{\xi}) + \lambda\left(D_{LNE}[\boldsymbol{\xi}:\boldsymbol{\xi}+d\boldsymbol{\xi}]-c\right) \\
&\approx \underset{d\boldsymbol{\xi}}{\arg \min}\ L(\boldsymbol{\xi}) + \partial_{\boldsymbol{\xi}}L(\boldsymbol{\xi})^\top d\boldsymbol{\xi}+\frac{\lambda}{2}d\boldsymbol{\xi}^\top g(\boldsymbol{\xi}) d\boldsymbol{\xi}-c\lambda.
\end{aligned}
\]
To solve this minimization, we set its derivative w.r.t. $d \boldsymbol{\xi}$ to zero:
\[
\begin{aligned}
0 &=\frac{\partial}{\partial d\boldsymbol{\xi}} L(\boldsymbol{\xi})+\partial_{\boldsymbol{\xi}} L(\boldsymbol{\xi})^\top d\boldsymbol{\xi}+\frac{\lambda}{2} d\boldsymbol{\xi}^\top \left[\delta - \tanh(\tau\boldsymbol{\xi})\tanh(\tau\boldsymbol{\xi})^\top\right] d\boldsymbol{\xi}-c\lambda \\
&=\partial_{\boldsymbol{\xi}} L(\boldsymbol{\xi})+\lambda \left[\delta - \tanh(\tau\boldsymbol{\xi})\tanh(\tau\boldsymbol{\xi})^\top\right] d\boldsymbol{\xi} \\
d\boldsymbol{\xi} &=-\frac{1}{\lambda} \left[\delta - \tanh(\tau\boldsymbol{\xi})\tanh(\tau\boldsymbol{\xi})^\top\right]^{-1} \partial_{\boldsymbol{\xi}} L(\boldsymbol{\xi})
\end{aligned}
\]
where a constant factor $1/\lambda$ can be absorbed into learning rate. Therefore, we get the optimal descent direction, i.e., the opposite direction of gradient, which takes into account the local curvature defined by $\left[\delta - \tanh(\tau\boldsymbol{\xi})\tanh(\tau\boldsymbol{\xi})^\top\right]^{-1}$.

\vskip 0.2in
\bibliography{sample}

\end{document}